\documentclass{article}

    \PassOptionsToPackage{numbers, compress}{natbib}

\usepackage[final]{neurips_2022}

\usepackage[utf8]{inputenc} 
\usepackage[T1]{fontenc}    
\usepackage{hyperref}       
\usepackage{url}            
\usepackage{booktabs}       
\usepackage{amsfonts}       
\usepackage{nicefrac}       
\usepackage{microtype}      
\usepackage{xcolor}         
\usepackage{lifetime}
\usepackage{tabularx}

\title{Domain Adaptation under Open Set Label Shift}

\author{%
Saurabh Garg,\, Sivaraman Balakrishnan,\, Zachary C. Lipton \\
Machine Learning Department, \\ 
Department of Statistics and Data Science, \\
Carnegie Mellon University\\
\small{\texttt{\{\href{mailto:sgarg2@andrew.cmu.edu}{sgarg2},\href{mailto:sbalakri@andrew.cmu.edu}{sbalakri},\href{mailto:zlipton@andrew.cmu.edu}{zlipton}\}@andrew.cmu.edu
}}}

\definecolor{cobalt}{rgb}{0.0, 0.28, 0.67}
\definecolor{carnelian}{rgb}{0.7, 0.11, 0.11}

\definecolor{bole}{rgb}{0.47, 0.27, 0.23}

\newcommand{\update}[1]{\textcolor{black}{#1}}

\usepackage[colorinlistoftodos,textwidth=3.0cm,textsize=tiny]{todonotes}

\begin{document}

\maketitle

\begin{abstract}
 We introduce the problem of domain adaptation 
under Open Set Label Shift (OSLS) \update{where}
the label distribution can change arbitrarily
and a new class may arrive during deployment,
but the class-conditional distributions $p(x|y)$ 
are domain-invariant.
OSLS subsumes domain adaptation under label shift  
and Positive-Unlabeled (PU) learning.  
The learner's goals here are two-fold:
(a) estimate the target label distribution,
including the novel class;
and (b) learn a target classifier. 
First, we establish necessary and sufficient conditions
for identifying these quantities.
Second, motivated by advances in label shift and PU learning,
we propose practical methods for both tasks
that leverage black-box predictors. 
Unlike typical open set domain adaptation problems,
which tend to be ill-posed and amenable only to heuristics, 
OSLS offers a well-posed problem
amenable to more principled machinery. 
Experiments across numerous semi-synthetic 
benchmarks on vision, language,
and medical datasets 
demonstrate that our methods
consistently outperform open set domain adaptation  baselines, 
achieving $10$--$25\%$ improvements 
in target domain accuracy. 
Finally, we analyze the proposed methods, 
establishing finite-sample convergence 
to the true label marginal and convergence to
optimal classifier for linear models in a Gaussian setup\footnote{Code is available at \url{https://github.com/acmi-lab/Open-Set-Label-Shift}.}.
\end{abstract}

\section{Introduction}
%
%
Suppose that we wished to deploy a machine learning system
to recognize diagnoses based on their clinical manifestations.
If the distribution of data were static over time,
then we could rely on the standard machinery of statistical prediction.
However, disease prevalences are constantly changing,
violating the assumption  of independent and identically distributed (iid) data. 
In such scenarios, we might reasonably 
apply the \emph{label shift} assumption,
where prevalences can change but
clinical manifestations cannot.
When only the relative proportion 
of previously seen diseases can change, 
principled methods can detect
and correcting for label shift on the fly
\citep{saerens2002adjusting, zhang2013domain, lipton2018detecting, azizzadenesheli2019regularized, alexandari2019adapting, garg2020labelshift}.
But what if a new disease, like COVID-19,
were to arrive suddenly?

Traditional label shift adaptation techniques break
when faced with a previously unseen class.
A distinct literature on Open Set Domain Adaptation (OSDA) seeks to handle such cases
\citep{panareda2017open, baktashmotlagh2018learning, cao2019learning, tan2019weakly, lian2019known, you2019universal, saito2018open, saito2020universal, fu2020learning}). 
Given access to labeled \emph{source} data and unlabeled \emph{target} data,
the goal in OSDA is to adapt classifiers
in general settings where previous classes
can shift in prevalence (and even appearance),
and novel classes separated 
out from those previously seen can appear.
Most work on OSDA is driven by the creation of 
and progress on benchmark datasets
(e.g., DomainNet, OfficeHome).
Existing OSDA methods are heuristic in nature,
addressing settings where the right answers 
seem intuitive but are not identified mathematically. 
However, absent assumptions on: 
(i) the nature of distribution shift 
among source classes and 
(ii) the relation between source 
classes and novel class, 
standard impossibility results for domain adaptation
condemn us to guesswork~\citep{ben2010impossibility}.
 
\begin{figure*}[t!]
    \centering 
    \subfigure{\includegraphics[width=\linewidth]{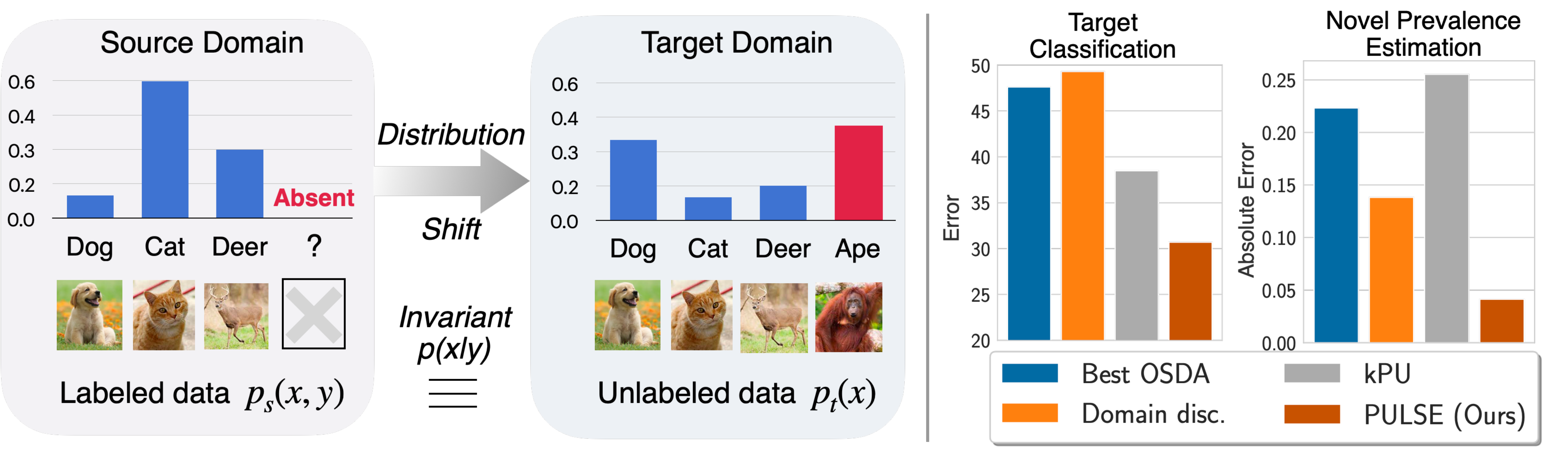}}
    \caption{\textbf{Left:} \emph{Domain Adaptation under OSLS}. An instantiation of OSDA that applies label shift assumption but allows for a new class to show up in target domain. \textbf{Right:} \emph{Aggregated results across seven semi-synthetic benchmark datasets}. For both target classification and novel class prevalence estimation, PULSE significantly outperforms other methods (lower error is better). For brevity, we only include result for the best OSDA method. For detailed comparison, refer \secref{sec:exp}.   
    }
    \label{fig:intro}
  \end{figure*}

In this work, we introduce domain adaptation
under Open Set Label Shift (OSLS), 
a coherent instantiation of OSDA 
that applies the label shift assumption
but allows for a new class to show up 
in the target distribution.
Formally, the label distribution may shift
between source and target $p_s(y) \neq p_t(y)$,
but the class-conditional distributions 
among previously seen classes may not
(i.e., $\forall y \in \{1,2, \dots,k\}, p_s(x|y) =p_t(x|y)$).
Moreover, a new class $y=k+1$ may arrive 
in the target period. 
Notably, OSLS subsumes label shift~\citep{saerens2002adjusting, storkey2009training, lipton2018detecting} (when $p_t(y=k+1)=0$)
and learning from Positive and Unlabeled (PU) data~\citep{de1999positive, letouzey2000learning, elkan2008learning} (when $k=1$). 
As with label shift and PU learning, 
our goals are two-fold.
Here, we must (i) estimate the target 
label distribution $p_t(y)$ (including the novel class \update{prevalence}); 
(ii) train a $(k+1)$-way target-domain classifier.

First, 
we characterize when the parameters of interest are identified
 (\secref{sec:identifiability}). 
Namely, we define a (necessary) \emph{weak positivity} condition,
which states that there exists a subset of each label's support
that has zero probability mass under the novel class
and that \update{the submatrix} of $p(x|y)$ consisting 
only of rows outside the novel class's support is full rank. 
Moreover, we prove that weak positivity alone is not sufficient. 
We introduce two sufficient conditions:
\emph{strong positivity} and \emph{separability},
either of which (independently) ensures identifiability. 

Focusing on cases with {strong positivity},
we show that OSLS reduces 
to $k$ PU learning problems (\secref{sec:reduction}). 
However, we demonstrate that straightforward applications of this idea 
fail because
(i) bias accumulates across the $k$ mixture proportion estimates
leading to grossly underestimating the novel class's prevalence;
and (ii) naive combinations of the $k$ PU classifiers are biased and inaccurate. 

Thus motivated, we propose the PULSE framework,
which combines methods from Positive and Unlabeled learning
and Label Shift Estimation, 
yielding two-stage techniques for both
label marginal estimation and classification (\secref{sec:framework}). 
Our methods build on recent advances in label shift
\citep{lipton2018detecting, azizzadenesheli2019regularized, alexandari2019adapting, garg2020labelshift} and PU learning \citep{kiryo2017positive,ivanov2019dedpul,garg2021PUlearning}, 
that leverage appropriately chosen black-box predictors  
to avoid the curse of dimensionality.
PULSE first estimates the label shift 
among previously seen classes,
and then re-samples the source data 
to formulate a single PU learning problem 
between (reweighted) source and target data
to estimate fraction of novel class and to learn the target classifier. 
In particular, our procedure builds on the BBE and CVIR 
techniques proposed in \citet{garg2021PUlearning}.
PULSE is simple to implement and compatible 
with arbitrary hypothesis classes (including deep networks).

We conduct extensive semi-synthetic experiments 
adapting seven benchmark datasets 
spanning vision (CIFAR10, CIFAR100, Entity30), 
natural language (Newsgroups-20), 
biology (Tabula Muris),
and medicine (DermNet, BreakHis)
(\secref{sec:exp}).
Across numerous data modalities, 
draws of the label distributions,
and model architectures,
PULSE consistently outperforms generic OSDA methods, 
improving by $10$--$25\%$ in accuracy on target domain. 
%
Moreover, PULSE outperforms methods 
that naively solve $k$ PU problems on both 
label distribution estimation and classification.

Finally, we analyze our framework (\secref{sec:analysis}).
First, we extend \citet{garg2021PUlearning}'s analysis of BBE
to derive finite-sample error bounds 
for our estimates of the label marginal.
Next, we develop new analyses of the CVIR objective \citep{garg2021PUlearning}
that PULSE relies in the classification stage. 
Focusing on a Gaussian setup and linear models
optimized by gradient descent,
we prove that CVIR converges
to a true positive versus negative classifier in population. 
Addressing the overparameterized setting 
where parameters exceed dataset size,
we conduct an empirical study
that helps to elucidate why, on separable data,
CVIR outperforms other consistent objectives,
including nnPU~\citep{kiryo2017positive} and uPU~\citep{du2015convex}.

\section{Related Work} \label{sec:related}
\textbf{(Closed Set) Domain Adaptation (DA) {} {}}
Under DA, the goal is to adapt a predictor
from a source distribution with labeled data
to a target distribution from which 
we observe only unlabeled examples. 
DA is classically explored under 
two distribution shift scenarios~\citep{storkey2009training}: 
(i) Covariate shift~\citep{zhang2013domain,zadrozny2004learning,cortes2010learning,cortes2014domain,gretton2009covariate}
where $p(y|x)$ remains invariant
among source and target; 
and (ii) Label shift~\citep{saerens2002adjusting, lipton2018detecting, azizzadenesheli2019regularized, alexandari2019adapting, garg2020labelshift, zhang2020coping} where $p(x|y)$
is shared across source and target.
%
%
In these settings most theoretical analysis 
requires that the target distribution's support 
is a subset of the source support~\citep{ben2010impossibility}. 
%
%
However, recent empirically work in 
DA~\citep{long2015learning,long2017deep, sun2016deep, sun2017correlation, zhang2019bridging, zhang2018collaborative, ganin2016domain, sohn2020fixmatch}     
focuses on settings 
motivated by benchmark datasets
(e.g., WILDS~\citep{ sagawa2021extending, wilds2021}, 
Office-31~\citep{saenko2010adapting}
OfficeHome~\citep{venkateswara2017deep}, 
DomainNet~\citep{peng2019moment}) 
where such overlap assumptions are violated. 
Instead, they rely on some intuitive
notion of semantic equivalence across domains. 
These problems are not well-specified
and in practice, 
despite careful hyperparameter tuning, 
these methods 
often do not improve over 
standard empirical risk minimization
on source data alone for practical, 
and importantly, previously unseen
datasets~\citep{sagawa2021extending}.

\textbf{Open Set Domain Adaptation (OSDA) {} {}} 
OSDA~\citep{panareda2017open, bendale2015towards, 6365193}
extends DA to settings where along with distribution 
shift among previously seen classes, 
we may observe a novel class in the target data. 
This setting is also known 
as \emph{universal domain adaptation}~\citep{you2019universal, saito2020universal}. 
\update{Rather than making precise assumptions 
about the nature of shift between source and target as in OSLS}, 
the OSDA literature is primarily governed by 
semi-synthetic problems on benchmark DA datasets
(e.g. DomainNet, Office-31 and OfficeHome). 
Numerous OSDA methods have been proposed~\citep{baktashmotlagh2018learning, cao2019learning, tan2019weakly, lian2019known, you2019universal, saito2018open, saito2020universal, fu2020learning, bucci2020effectiveness}. 
At a high level, most OSDA methods perform two steps: 
(i) align source and target representation
for previously seen classes; and (ii) train a 
discrimination to reject 
novel class from previously seen classes. 
The second step typically uses 
novelty detection heuristics
to identify  novel samples.

\textbf{PU Learning {} {}} 
\update{Positive and Unlabeled (PU) learning 
is the base case of OSLS}.
Here, we observe labeled data a single source class  
and unlabeled target data contains 
data from both the novel class and the source class. 
In PU learning, our goals are: 
(i) Mixture Proportion Estimation (MPE), i.e., 
determining the fraction of previously seen class 
in target ; and (ii) PU classification, 
i.e., learning to discriminate between
the novel and the positive (source) class.
Several classical methods were proposed 
for both MPE~\citep{elkan2008learning,du2014class,scott2015rate,jain2016nonparametric,bekker2018estimating,bekker2020learning} 
and classification~\citep{elkan2008learning, du2014analysis,du2015convex}. 
However, classical MPE methods
do not scale to high-dimensional settings~\citep{ramaswamy2016mixture}. 
More recent methods alleviate these issues
by operating in classifier output space~\citep{garg2021PUlearning, ivanov2019dedpul}. 
For classification, traditional methods fail 
when deployed with models classes with high capacity
due to their capacity of fitting random labels~\citep{zhang2016understanding}. 
Recent methods~\citep{garg2021PUlearning, kiryo2017positive, chen2020self}, 
avoid over-fitting by employing regularization 
or self-training techniques.


%


\textbf{Other related work {} {}}
%
A separate line of work looks at the problem 
of Out-Of-Distribution (OOD) detection
\citep{hendrycks2016baseline,geifman2017selective,lakshminarayanan2016simple, jiang2018trust, ovadia2019can, zhang2020hybrid}. 
Here, the goal is to identify novel examples, 
i.e., samples that lie out of the support of training distribution. 
The main different between OOD detection and OSDA 
is that in OOD detection we do not have access to 
unlabeled data containing a novel class. 
Recently, \citet{cao2021open} proposed 
open-world semi-supervised learning,
where the task is to not only identify 
novel classes in target but also to separate out different 
novel classes in an unsupervised manner. 

\update{Our work takes a step back 
from the hopelessly general OSDA setup,
introducing OSLS, a well-posed OSDA setting
where the sought-after parameters can be identified.}

\section{Open Set Label Shift} \label{sec:setup}
%
%

\textbf{Notation {} {}} 
For a vector $v\in \Real^d$, 
we use $v_j$ to denote its $j^\text{th}$ entry, 
and for an event $E$, we let $\indict{E}$ 
denote the binary indicator of the event.
By $\abs{A}$, we denote the cardinality of set $A$.  

Let $\inpt \in \Real^d $ be the input space 
and $\out = \{ 1,2, \ldots, k+1\}$ be the output space
for multiclass classification.
Let $\ProbS$ and $\ProbT$ be the 
source and target distributions and 
let $\ps$ and $\pt$ denote the corresponding 
probability density (or mass) functions. 
By $\mathbb{E}_s$ and $\mathbb{E}_t$, 
we denote expectations over 
the source and target distributions.
We assume that we are given a loss function 
$\lossell: \Delta^{k} \times \out \to \Real $,
such that $\ell(z, y)$ is the loss incurred 
by predicting $z$ when the true label is $y$. 
Unless specified otherwise,
we assume that $\ell$ is the cross entropy loss. 
As in standard unsupervised domain adaptation, 
we are given independently and identically distributed (iid) 
samples from labeled source data 
$\{(\x_1,y_1), (\x_2, y_2), \ldots , (\x_n, y_n)\}\sim \ProbS^n$ 
and iid samples from unlabeled target data
$\{\x_{n+1}, \x_{n+2}, \ldots, \x_{n+m}\}\sim \ProbT^m$.  

Before formally introducing OSLS, we describe 
label shift and PU learning settings. 
Under label shift, we observe data 
from $k$ classes in both source 
and target where the conditional 
distribution remain invariant 
(i.e., $p_s(x|y) = p_t(x|y)$ 
for all classes $y \in [1,k]$) 
but the target label marginal  
may change (i.e., $p_t(y) \ne p_s(y)$). 
Additionally, for all classes 
in source have a non-zero support , i.e., 
for all $y\in [1, k]$, $p_s(y) \ge c$, where $c > 0$. 
Under PU learning, we possess labeled source data 
from a positive class and unlabeled target data 
from a mixture of positive and negative class
with a goal of learning a positive-versus-negative classifier
on target. 
We now introduce the OSLS setting:


\begin{definition}[Open set label shift]
Define $\out_t = \out$ and 
$\out_s = \out \setminus \{k+1\}$. 
Under OSLS, the label distribution among 
source classes $\out_s$ may change
but the class conditional $p(x|y)$ 
for those classes remain invariant 
between source and target, 
and the target domain may contain a novel class, i.e., 
\begin{align}
    \ps(x|y = j) = \pt(x|y=j) \quad \forall j \in \out_s \qquad \text{and} \qquad \ps(y = k+1) = 0 \,. \label{eq:OSLS}
\end{align}
Additionally, we have non-zero support 
for all $k$ (previously-seen) labels 
in the source distribution, 
i.e., for all $y \in \out_s$, 
$\ps(y) \ge c$ for some $c > 0$. 
\end{definition}

Note that the label shift and PU learning problems
can be obtained as special cases of OSLS.  
When no novel class is observed in target 
(i.e., when $p_t(y = k+1) = 0$),
we recover the label shift problem,
and when we observe only one class in source 
(i.e., when $k=1$), 
the OSLS problem reduces to PU learning.
%
Under OSLS, our goal naturally breaks 
down into two tasks: 
(i) estimate the target label marginal $p_t(y)$ 
for each class $y \in \out$; 
(ii) 
train a classifier $f: \inpt \to \Delta^k$ 
to approximate $p_t(y|x)$. 
\section{Identifiablity of OSLS} \label{sec:identifiability} 
We now introduce conditions for OSLS,
under which the solution is identifiable. 
Throughout the section, 
we will assume access to population 
distribution for labeled source data and unlabeled target data, 
i.e., $p_s(x,y)$ and $p_t(x)$ is given. To keep
the discussion simple, we assume finite input domain $\inpt$
which can then be relaxed to continuous inputs. 
We relegate proofs to \appref{sec:proof_identi_app}. 

We first make a connection between target 
label marginal $\pt(y)$ estimation 
and learning the target classifier $\pt(y|x)$
showing that recovering $\pt(y)$ is 
enough to identify $p_t(y|x)$. 
\update{In population, given access to $p_t(y)$, the 
class conditional $\pt(x|y=k+1)$ can be 
obtained in closed form as $\nicefrac{\left(\pt(x) - \sum_{j=1}^k \pt(y=j) \ps(x|y=j)\right)}{p_t(y = k+1)}$.} 
We can then apply Bayes rule to obtain $\pt(y|x)$.  
Henceforth, we will focus our discussion on identifiability 
of $\pt(y)$ which implies identifiability of $p_t(y|x)$. 
In following proposition, we present \emph{weak positivity},
a necessary condition for $p_t(y)$ to be identifiable. 
%
%
\begin{restatable}[Necessary conditions]{reprop}{necessary} \label{prop:necessary}
Assume $p_t(y) > 0$ for all $y\in \out_t$. Then 
$p_t(y)$ is identified only if $\pt(x|y=k+1)$ and 
$\ps(x|y)$ for all $y\in \out_s$ satisfy weak positivity, i.e., there must  
exists a subdomain $X_\wpos \subset X$ such that:
\begin{enumerate}[(i)]
  \setlength{\itemsep}{3pt}
  \setlength{\parskip}{0pt}
  \setlength{\parsep}{0pt}
    \item  $\pt(X_\wpos | y = k+1) = 0$; and 
    \item the matrix $\left[\ps(x|y)\right]_{x\in X_\wpos, y\in \out_s}$ is full column-rank. 
\end{enumerate}
\end{restatable}

Intuitively, \propref{prop:necessary} states 
that if the target marginal doesn't lie
on the vertex of the simplex $\Delta^k$,
then their must exist a subdomain $X_\wpos$ 
where the support of novel class is zero 
and within $X_\wpos$, 
$\pt(y)$ for source classes is identifiable. 
While it may seem that existence of 
a subdomain $X_\wpos$ is enough, 
we show that for the OSLS problem, 
existence doesn't imply uniqueness. 
In \appref{subsec:counter}, 
we construct an example,
where the weak positivity condition is not sufficient. 
In that example, we show that there can 
exist two subdomains $X_\wpos$ 
and $X_\wpos^\prime$ satisfying weak positivity, 
both of which lead to separate solutions for $\pt(y)$.
Next, we extend weak positivity to 
two stronger conditions, 
either of which (alone) implies identifiability. 
\begin{restatable}[Sufficient conditions]{reprop}{sufficient} \label{prop:sufficiency}
The target marginal $p_t(y)$ is identified 
if for all $y\in \out \setminus \{k+1\}$, 
$\pt(x|y=k+1)$ and  $\ps(x|y)$ 
satisfy either:
\begin{enumerate}[(i)]
  \setlength{\itemsep}{3pt}
  \setlength{\parskip}{0pt}
  \setlength{\parsep}{0pt}
    \item Strong positivity, i.e.,
there exists $X_\spos \subset \inpt$ 
such that $p_t(X_\spos| y= k+1) = 0$ 
and the matrix $\left[\ps(x|y)\right]_{x\in X_\spos, y\in \out_s}$
is full-rank and diagonal; or
    \item Separability, i.e., there exists $X_\sep \subset \inpt$,
such that $p_t(X_\sep | y= k+1 ) = 0\,$, $p_s(X_\sep) = 1\,,$ 
and the matrix $\left[\ps(x|y)\right]_{x\in X_\sep, y\in \out_s}$ 
is full column-rank. 
\end{enumerate}
\end{restatable}
Strong positivity generalizes the irreducibility 
condition~\citep{blanchard2010semi}, 
which is sufficient for identifiability under PU learning,
to $k$ PU learning problems. 
Note that while the two conditions 
in \propref{prop:sufficiency} overlap, 
they cover independent set of OSLS problems.  
Informally, strong positivity extends weak positivity 
by making an additional assumption that the matrix 
formed by $p(x|y)$ on inputs in $X_\wpos$ is diagonal 
and the separability assumption
extends the weak positivity condition 
to the full input domain of source classes
instead of just $X_\wpos$. 
Both of these conditions identify 
a support region of $\inpt$ 
which purely belongs to source classes 
where we can either individually estimate 
the proportion of each source classes 
(i.e., under strong positivity) 
or jointly estimate the proportion 
(i.e., under separability). 

%

To extend our identifiability conditions 
for continuous distributions, 
the linear independence 
conditions on the matrix 
$\left[\ps(x|y)\right]_{x\in X_\sep, y\in \out_s}$
has the undesirable property of being 
sensitive to changes on sets of measure zero.
We may introduce stronger notions 
of linear independence as in Lemma 1 
of \citet{garg2020labelshift}.
We discuss this in \appref{subsec:extend_ident_continuos}.


\section{Reduction of OSLS to $k$ PU Problems} \label{sec:reduction}
Under the strong positivity condition, 
the OSLS problem can be broken down
into $k$ PU problems as follows: 
By treating a given source class $y_j \in \out_s$
as \emph{positive} and grouping
all other classes together as \emph{negative}
we observe that the unlabeled target data 
is then a mixture of data from 
the positive and negative classes.
This yields a PU learning problem 
and the corresponding mixture proportion 
is the 
fraction $\pt(y = j)$ (proportion of class $y_j$) among the target data.
By iterating this process for all source classes,
we can solve for the entire target label marginal $\pt(y)$.
Thus, OSLS reduces to $k$ instances of PU learning problem. 
Formally, note that $\pt(x)$ can be written as:
\begin{align}
\pt(x) = \pt(y=j) \ps(x|y = j) + \left(1- \pt(y=j)\right) \left(\sum\nolimits_{ i \in \out \setminus \{j\}} \mfrac{\pt(y = i)}{ 1- \pt(y = j)} \ps(x|y = i)\right) \label{eq:reduction}\,, 
\end{align}
individually for all $j \in \out_s$. 
By repeating this reduction for all classes,
we obtain $k$ separate PU learning problems. 
Hence, a natural choice is to leverage 
this structure and solve $k$ PU problems
to solve the original OSLS problem. 
\update{In particular, for each class $j \in \out_s$, 
we can first estimate its prevalence $\hat \pt(y = j)$
in the unlabeled target. 
Then the target marginal for the novel class is given by 
$\smash{\hat \pt( y = k+1) = 1 - \sum_{i=1}^k \hat \pt (y = i) }$.}
Similarly, for classification,
we can train $k$ PU learning classifiers $f_i$,
where $f_i$ is trained to classify a source class $i$ 
versus others in target. 
An example is classified as belonging to the class
$y=k+1$, if it rejected by all classifiers $f_i$ as other in target. 
We explain this procedure more formally in \appref{subsec:reduction}. 

This reduction has been mentioned in past work
\citep{sanderson2014class, xu2017multi}. 
However, to the best of our knowledge, 
no previous work has empirically investigated 
both classification and target label marginal 
estimation jointly. 
\citet{sanderson2014class} 
focuses only on target marginal estimation 
for 
tabular datasets 
and \citet{xu2017multi} assumes that 
the target marginal is known and only 
trains $k$ separate PU classifiers. 

In our work, 
%
we perform the first large scale experiments 
to evaluate efficacy of the reduction 
of the OSLS problem to $k$-PU problems.
With plugin state-of-the-art PU learning algorithms, 
%
we observe that this naive reduction doesn't scale 
to datasets with large number of classes
because of error accumulation in each 
of the $k$ MPEs 
and $k$ one-versus-other PU classifiers.
To mitigate the error accumulation problem, 
we propose the PULSE framework in the next section. 

\section{The PULSE Framework for OSLS 
} \label{sec:framework}

We begin with presenting our framework for OSLS problem 
under strong positivity condition. 
First, we explain the structure of OSLS 
that we leverage in PULSE framework
and then elaborate design decisions 
we make
to exploit the identified structure.


\paragraph{Overview of PULSE framework}
Rather than simply dividing each 
OSLS instance into $k$ PU problems, 
we exploit the joint structure 
of the problem to obtain a \emph{single} PU learning problem. 
To begin, we note that if only 
we could apply a
\emph{label shift correction}
to source, i.e.,
re-sample source classes 
according to their relative proportion
in the target data,
then we could subsequently consider 
the unlabeled target data 
as a mixture of (i)
the (reweighted) source distribution;
and (ii) the novel class distribution 
(i.e., $p_t(x|y=k+1)$). 
\update{Formally, we have 
\begin{align*}
    p_t(x) = \sum_{j \in \out_t} p_t(y=j) & p_t(x|y=j)  = 
    \sum_{j \in \out_s} \mfrac{p_t(y=j)}{p_s(y=j)} p_s(x, y=j) + p_t(x|y=k+1) p_t(y=k+1) \\
    &= (1- p_t(y=k+1)) p_s^\prime(x) + p_t(y=k+1) p_t(x|y=k+1)\,, \numberthis \label{eq:label_shift_corr}
\end{align*}
where $p_s^\prime(x)$ is the label-shift-corrected source distribution, 
i.e., $p_s^\prime(x) = \sum_{j \in \out_s} w(j) p_s(x,y=j)$,
where $w(j) = {\left(\nicefrac{p_t(y=j)}{\sum_k p_t(y=k)}\right)}/{p_s(y=j)}$ 
for all $j \in \out_s$. 
Intuitively, $p_t^\prime(j) = \nicefrac{p_t(y=j)}{\sum_k p_t(y=k)}$ is
re-normalized label distribution in target 
among source classes 
and $w(j)$'s are the importance weights.} 
%
Hence, after applying a label shift correction
to the source distribution $p_s^\prime(x)$, 
we have reduced the OSLS problem 
to a \emph{single} PU learning problem,
where $p_s^\prime(x)$ plays the part 
of the positive distribution and $p_t(x| y=k+1)$ acts
as negative distribution with mixture coefficients
$1 - p_t(y=k+1)$ and $p_t(y=k+1)$ respectively. 
\update{We now discuss our methods (i) to estimate 
the importance ratios $w(y)$; 
and (ii) to tackle the PU learning instance obtained from OSLS.  
}

\paragraph{Label shift correction: Target marginal estimation among source classes} While traditional methods for estimating label shift 
breakdown in high dimensional settings~\citep{zhang2013domain}, 
recent methods exploit black-box classifiers
to avoid the curse of dimensionality~\citep{lipton2018detecting, azizzadenesheli2019regularized, alexandari2019adapting}. 
However, these recent techniques
require overlapping label distributions, 
and a direct application would require
demarcation of samples from $p^\prime_s(x)$ \update{sub-population} 
in target, creating a cyclic dependency. 
\update{Instead, to estimate the relative proportion of previously seen classes in target, we leverage the $k$ PU reduction described in \secref{sec:reduction} with two crucial distinctions. First, we normalize the obtained estimates of fraction previously seen classes to obtain the relative proportions in $\ps^\prime(y)$. In particular, we do not leverage the estimates of previously seen class proportions in target to directly estimate the proportion of novel class which avoids issues due to error accumulation. 
Second, we exploit a $k$-way source classifier $f_s$ trained on labeled source data instead of training $k$ one-versus-other PU classifiers. }
\update{
We tailor the recently proposed 
Best Bin Estimation (BBE) 
technique from \citet{garg2021PUlearning}. 
%
We describe the modified BBE procedure in 
\appref{sec:OSLS_framework_app} (\algoref{alg:BBE}). 
After estimating the relative fraction 
of source classes in target (i.e., $\wh p_t^\prime(j) =\nicefrac{ \wh p_t(y=j)}{\sum_{k\in \out_s} \wh p_t(y=k)}$ for all $j \in \out_s$), 
we re-sample the source data according to $\wh p_t^\prime(y)$
to mimic samples from distribution $p_s^\prime(x)$. 
}

\begin{algorithm}[t!]
  \caption{Positive and Unlabeled learning post Label Shift Estimation (PULSE) framework}
  \label{alg:PULSE}
  \begin{algorithmic}[1]
  \INPUT: Labeled source data $\{{\bf X}^S, {\bf y}^S\}$ and unlabeled target samples ${\bf X}^T$. 
    \STATE Randomly split data into training $\{{\bf X}^S_1, {\bf y}^S_1\}$, ${\bf X}^T_1$ and hold out partition $\{{\bf X}^S_2, {\bf y}^S_2\}$, ${\bf X}^T_2$.
    \STATE Train a source classifier $f_s$ on labeled source data $\{{\bf X}^S_1, {\bf y}^S_1\}$.
    \STATE \update{Estimate label shift $\wh p^\prime_t(y = j) = \mfrac{\wh p_t(y=j)}{\sum_{k\in \out_s} \wh p_t(y=k)}$ using \algoref{alg:BBE}} and hence importance ratios $\wh w(j)$ among source classes $j \in \out_s$. 
    \STATE Re-sample training source data according to label distribution  $\wh p^\prime_t$ to get $\{ \widetilde{\bf X}^S_1, \widetilde{\bf y}^S_1\}$ and $\{ \widetilde{\bf X}^S_2, \widetilde{\bf y}^S_2\}$. 
    \STATE Using \algoref{alg:CVIR}, train a discriminator $f_d$ and estimate novel class fraction $\wh p_t(y=k+1)$. 
    \STATE Assign $[f_t(x)]_j = (f_d(x)) \mfrac{\wh w(j) \cdot [f_s(x)]_j}{\sum_{k\in \out_s} \wh w(k) \cdot [f_s(x)]_k} $ for all $j \in \out_s$ and $[f_t(x)]_{k+1} = 1 - f_d(x)$. And for all $j \in \out_s$, assign $\wh p_t(y = j) = (1-\wh p_t(y=k+1))\cdot \wh p_t^\prime(y = j)$. 
  \OUTPUT: Target marginal estimate $\wh p_t \in \Delta^k$ and target classifier $f_t (\cdot) \in \Delta^k$.  
\end{algorithmic}
\end{algorithm}

\paragraph{PU Learning: Separating the novel class from previously seen classes}
After obtaining a PU learning problem instance, 
we resort to PU learning techniques to 
(i) estimate the fraction of novel class $p_t(y=k+1)$; 
and (ii) learn a binary classifier $f_d(x)$ to discriminate 
between label shift corrected source $p_s^\prime(x)$
and novel class $p_t(x|y=k+1)$. 
\update{
%
With traditional methods for PU learning
involving domain discrimination, 
over-parameterized models can
memorize the positive instances 
in unlabeled,
assigning them confidently 
to the negative class,
which can severely hurt
generalization on PN 
data~\citep{kiryo2017positive, garg2021PUlearning}.
Rather, we employ Conditional 
Value Ignoring Risk (CVIR) loss
proposed in \citet{garg2021PUlearning} 
which was shown to outperform alternative 
approaches.
First, we estimate the proportion 
of novel class $\wh p_t(y=k+1)$ with BBE.
Next, given an estimate $\wh p_t(y=k+1)$,
CVIR objective discards
the highest loss 
$(1 - \wh p_t(y=k+1))$ fraction of examples
on each training epoch, removing the 
incentive to overfit to
the examples from $p_s^\prime(x)$. 
%
Consequently, we employ the 
iterative procedure that alternates between 
estimating the prevalence 
of novel class $\wh p_t(y=k+1)$ (with BBE) and 
minimizing the CVIR loss with estimated 
fraction of novel class. 
We detail this procedure in \appref{sec:OSLS_framework_app} (\algoref{alg:CVIR}).  
}

\paragraph{Combining PU learning and label shift correction} 
Finally, to obtain a $(k+1)$-way classifier $f_t(x)$ on target 
we combine discriminator $f_d$ and  source classifier $f_s$
with importance-reweighted label shift correction. 
In particular,
for all $j \in \out_s$, $[f_t(x)]_j = (f_d(x)) \frac{w(j) \cdot [f_s(x)]_j}{\sum_{k\in \out_s} w(k) \cdot [f_s(x)]_k} $
and $[f_t(x)]_{k+1} = 1 - f_d(x)$. 
\update{Overall, our approach outlined in \algoref{alg:PULSE} proceeds
as follows:
First, we estimate the label shift among
previously seen classes. 
Then we employ importance re-weighting of source 
data to formulate a single PU learning problem 
to estimate the fraction 
of novel class $\wh p_t(y=k+1)$ and 
to learn a discriminator $f_d$ 
for the novel class. Combining discriminator 
and label shift corrected source classifier
we get $(k+1)$-way target classifier. 
We analyse crucial steps in PULSE in 
\secref{sec:analysis}.} 
%

Our ideas for PULSE framework 
can be
extended to separability 
condition since   
\eqref{eq:label_shift_corr}
continues to hold. 
However, in our initial experiments,
we observe that techniques proposed 
under strong positivity were empirically 
stable and outperform methods
developed under separability.
This is intuitive for many benchmark datasets 
where it is natural to assume that 
for each class there exists a 
subdomain that only belongs to that class. 
We describe this in more detail in \appref{subsec:pulse_separablity_app}. 
\section{Experiments} \label{sec:exp}

\textbf{Baselines {} {}} 
We compare PULSE with several popular methods from OSDA literature. While these methods are not specifically 
proposed for OSLS, they are introduced for the more general OSDA problem. In particular, we make comparions with DANCE~\citep{saito2020universal}, UAN~\citep{you2019universal}, CMU~\citep{fu2020learning}, STA~\citep{liu2019separate}, Backprop-ODA (or BODA)~\citep{saito2018open}. We use the open source implementation available at \url{https://github.com/thuml}.
For alternative baselines, we experiment with source classifier directly deployed on the target data which may contain novel class and label shift among source classes (referred to as \emph{source-only}). We also train a domain discriminator classifier for source versus target (referred to as \emph{domain disc.}). This  is adaptation of PU learning baseline\citep{elkan2008learning} which assumes no label shift among source classes.
Finally, per the reduction presented in \secref{sec:reduction}, we train $k$ PU classifiers (referred to as \emph{k-PU}).
We include detailed description of each method in \appref{subsec:baselines_app}.

\textbf{Datasets {} {}} We conduct experiments with seven benchmark classification datasets across 
vision, natural language, biology and medicine. For each dataset, we simulate an OSLS problem as described in next paragraph.  
For vision, we use CIFAR10, CIFAR100 ~\citep{krizhevsky2009learning} and Entity30~\citep{santurkar2020breeds}. 
%
For language, we experiment with Newsgroups-20~(\url{http://qwone.com/~jason/20Newsgroups/})
dataset. 
%
Additionally, inspired by applications of OSLS in biology and medicine, we experiment with Tabula Muris~\citep{tabula2020single} (Gene Ontology prediction), Dermnet (skin disease prediction~\url{https://dermnetnz.org/}), and BreakHis~\citep{spanhol2015dataset} (tumor cell classification). 
These datasets span language, image and table modalities. We provide interpretation of OSLS problem for each dataset along with other details in \appref{subsec:datasets_app}. 

\textbf{OSLS Setup {} {}} To simulate an OSLS problem, we experiment with different fraction of novel class prevalence, source label distribution, and target label distribution. We randomly choose classes that constitute the novel target class. 
After randomly choosing source and novel classes, we first split the training data from each source class 
randomly into two partitions. This creates a random label distribution for shared classes among source and target. We then club novel classes to assign them a new class (i.e. $k+1$). 
Finally, we throw away labels for the target data to obtain an unsupervised DA problem. We repeat the same process on iid hold out data to obtain validation data with no target labels. 

\textbf{Training and Evaluation {} {} } We use Resnet18~\citep{he2016deep} for CIFAR10, CIFAR100, and Entity30. For newsgroups, we use a convolutional architecture. For Tabular Muris and MNIST, we use a fully connected MLP. For Dermnet and BreakHis, we use Resnet-50. For all methods, we use the same backbone for discriminator and source classifier. For kPU, we use a separate final layer for each class with the same backbone. We use default hyperparameters for all methods. 
For OSDA methods, we use default method specific hyperparameters introduced in their works. 
Since OSDA methods do not estimate the prevalence of novel class explicitly,
we use the fraction of examples predicted in class $k+1$ as a surrogate. 
We train models till the performance on validation source data (labeled) ceases to increase. Unlike OSDA methods, note that we do not use early stopping based on performance on held-out labeled target data.  
To evaluate classification performance, we report target accuracy on all classes, seen classes and the novel class. For novel class prevalence estimation, we report absolute difference between the true and estimated marginal. 
We open-source our code and 
by simply changing a single config file, new OSLS setups can be generated and experimented with. We provide precise details about hyperparameters, OSLS setup for each dataset and code in \appref{subsec:setup_app}.


\begin{table}[t]
  \centering
  \small
  \setlength{\tabcolsep}{2pt}
  \renewcommand{\arraystretch}{1.2}
  \caption{\emph{Comparison of PULSE with other methods}. Across all datasets, PULSE outperforms alternatives for both target classification and novel class prevalence estimation. 
  Acc~(All) is target accuracy, Acc~(Seen) is target accuracy on examples from previously seen classes, and Acc~(Novel) is recall for novel examples. MPE~(Novel) is absolute error for novel prevalence estimation.  
  Results  reported by averaging across 3 seeds.   
    Detailed results for each dataset with all methods in \appref{subsec:results_paper_app}.}\label{table:results}
  \vspace{5pt}
  \begin{tabular}{@{}*{11}{c}@{}}
  \toprule
   & & \multicolumn{4}{c}{\textbf{CIFAR-10}} & & \multicolumn{4}{c}{\textbf{CIFAR-100}}   \\
    \multirow{2}{*}{Method} & &   \multirow{2}{*}{  \parbox{1.cm}{\centering Acc (All)}} & \multirow{2}{*}{  \parbox{1.cm}{\centering Acc (Seen)}} & \multirow{2}{*}{  \parbox{1.cm}{\centering  Acc (Novel)}} & \multirow{2}{*}{  \parbox{1.cm}{\centering MPE (Novel)}} & &  \multirow{2}{*}{  \parbox{1.cm}{\centering Acc (All)}} & \multirow{2}{*}{  \parbox{1.cm}{\centering Acc (Seen)}} & \multirow{2}{*}{  \parbox{1.cm}{\centering Acc (Novel)}} & \multirow{2}{*}{  \parbox{1.cm}{\centering MPE (Novel)}} \\
    & & & & & & & & \\
  \midrule
  Source-Only && ${67.1}$ & $87.0$ & - & - && $46.6$ & $66.4$ & - & - \\
  \midrule 
  UAN~\citep{you2019universal}   && $15.4$ & $19.7$ & $25.2$ & $0.214$ && $18.1$ & $40.6$ & $14.8$ & $0.133$\\
  BODA~\citep{saito2018open} && $63.1$ & $66.2$ & $42.0$ & $0.162$ && $36.1$ & $17.7$ & $81.6$ & $0.41$\\ 
  DANCE~\citep{saito2020universal} && $70.4$ & $85.5$ & $14.5$ & $0.174$ && $47.3$ & $66.4$ & $1.2$ & $0.28$  \\
  STA~\citep{liu2019separate} && $57.9$ & $69.6$ & $14.9$ & $0.124$ && $42.6$ & $48.5$ & $34.8$ & $0.14$\\
  CMU~\citep{fu2020learning} && $62.1$ & $77.9$ & $41.2$ & $0.183$ && $35.4$ & $46.0$ & $15.5$ & $0.161$\\
  \midrule
  Domain Disc.~\citep{elkan2008learning} && $47.4$ & $87.0$ & $30.6$ & $0.331$ && $45.8$ & $66.5$ & $39.1$ & $\bf{0.046}$\\
  $k$-PU && $83.6$ & $79.4$ & $\bf{98.9}$ & $0.036$ && $36.3$ & $22.6$ & $\bf{99.1}$ & $0.298$  \\ 
  PULSE (Ours) && $\bf{86.1}$ & $\bf{91.8}$ & $88.4$ & $\bf{0.008}$ && $\bf{63.4}$ & $\bf{67.2}$ & $63.5$ & $0.078$ \\ 
  \bottomrule 
  \end{tabular}  
\end{table}

\begin{table}[t]
  \centering
  \small
  \setlength{\tabcolsep}{1.5pt}
  \renewcommand{\arraystretch}{1.2}
  \begin{tabular}{@{}*{16}{c}@{}}
  \toprule
   &   
   \multicolumn{2}{c}{\textbf{Entity30}} & 
   \multicolumn{2}{c}{\textbf{Newsgroups20}}  &  
   \multicolumn{2}{c}{\textbf{Tabula Muris}} &  
   \multicolumn{2}{c}{\textbf{BreakHis}} & 
   \multicolumn{2}{c}{\textbf{DermNet}}   \\
    \multirow{2}{*}{Method} 
     &    \multirow{2}{*}{  \parbox{1.0cm}{\centering  Acc (All)}}  &  \multirow{2}{*}{  \parbox{1.cm}{\centering MPE (Novel)}} 
      &   \multirow{2}{*}{  \parbox{1.0cm}{\centering  Acc (All)}}  &  \multirow{2}{*}{  \parbox{1.cm}{\centering MPE (Novel)}} 
    &    \multirow{2}{*}{  \parbox{1.0cm}{\centering  Acc (All)}}  &  \multirow{2}{*}{  \parbox{1.cm}{\centering MPE (Novel)}} 
    &  \multirow{2}{*}{  \parbox{1.0cm}{\centering  Acc (All)}}  & \multirow{2}{*}{  \parbox{1.cm}{\centering MPE (Novel)}}
    &  \multirow{2}{*}{  \parbox{1.0cm}{\centering  Acc (All)}} & \multirow{2}{*}{  \parbox{1.cm}{\centering MPE (Novel)}} \\
    & & & & & & & & & \\
  \midrule
  Source-Only & $32.0$ & - & $39.3$ & - & $33.8$ &  - & $70.0$ & - & $41.4$  & - \\
  \midrule 
  BODA~\citep{saito2018open} & $42.2$  & $0.189$ & $43.4$ & $0.16$ & $76.5$ & $0.079$ &  $71.5$ &  $0.077$ & $43.8$ & $0.207$\\
  \midrule
  Domain Disc. & $43.2$ & $0.135$ & $50.9$ & $0.176$& $73.0$ & $0.071$ & $56.5$& $0.091$ & $40.6$  & $0.083$\\
  $k$-PU & $50.7$ & $0.394$  & $52.1$ & $0.373$ & $85.9$ & $0.307$ & $75.6$ & $\bf{0.059}$ & $46.0$  & $0.313$ \\ 
  PULSE (Ours) & $\bf{58.0}$ & $\bf{0.054}$  & $\bf{62.2}$ & $\bf{0.061}$ & $\bf{87.8}$ & $\bf{0.058}$ & $\bf{79.1}$ & $\bf{0.054}$ & $\bf{48.9}$ & $\bf{0.043}$\\ 
  \bottomrule 
  \end{tabular}  
\end{table}

\textbf{Results {} {}} Across different datasets, we observe that PULSE consistently outperforms other methods for the target classification and novel prevalence estimation (\tabref{table:results}). For detection of novel classes (Acc~(Novel) column), kPU achieves superior performance  
as compared to alternative approaches because of its bias to default to $(k+1)^\text{th}$ class. This is evident by the sharp decrease in performance on previously seen classes.  
For each dataset, we plot evolution of performance with training in~\appref{subsec:results_paper_app}. We observe 
more stability in performance of PULSE as compared to other methods. 
%


We observe that with
default hyperparameters, 
popular OSDA methods significantly 
under perform as compared to PULSE. 
We hypothesize that the primary reasons 
underlying the poor performance 
of OSDA methods are (i) the heuristics 
employed to detect novel classes; and 
(ii) loss functions incorporated to 
improve alignment between examples from 
common classes in source and target.
To detect novel classes, 
a standard heuristic employed in 
popular OSDA methods involves 
thresholding uncertainty 
estimates (e.g., prediction entropy, 
softmax confidence~\citep{you2019universal,fu2020learning,saito2020universal}) 
at a predefined 
threshold $\kappa$. However, 
a fixed $\kappa$, may not 
for different datasets and 
different fractions of 
the novel class. In \appref{subsec:OSDA_app}, 
we ablate by (i) removing loss function terms 
incorporated with an aim to improve source 
target alignment; and (ii) vary threshold
$\kappa$ and show improvements in performance 
of these methods. 
In contrast, our two-stage method 
PULSE, first estimates 
the fraction of novel class which 
then guides the classification 
of novel class versus previously seen classes
avoiding the need to guess $\kappa$. 

\textbf{Ablations {} {}}  
Different datasets, in our setup span different fraction of novel class prevalence ranging from $0.22$ (in CIFAR10) to $0.64$ (in Tabula Muris). For each dataset, we perform more ablations on the novel class proportion in \appref{subsec:novel_class_app}.
For kPU and PULSE, in the main paper, we include results with BBE and CVIR~\citep{garg2021PUlearning}. In \appref{subsec:different_PU_app}, we perform experiments with alternative PU learning approaches and highlight the superiority of BBE and CVIR over other methods. 
Moreover, since we have access to unlabeled target data, we experiment with SimCLR~\citep{chen2020simple} pre-training on the mixture of unlabeled source and target dataset. We include setup details and results in \appref{subsec:contrastive_app}. While pre-trained backbone architecture improves performance for all methods, PULSE continues to dominate other methods.

\section{Analysis of PULSE Framework} \label{sec:analysis}
In this section, we analyse 
key steps of our PULSE procedure
for target label marginal estimation 
(Step 3, 5 \algoref{alg:PULSE})
and learning the domain discriminator 
classifier (Step 5, \algoref{alg:PULSE}). 
Due to space constraints, we present informal results here 
and relegate formal statements and proofs to \appref{sec:OSLS_theory_app}.  

\paragraph{Theoretical analysis for target marginal estimation}
Building on BBE results from \citet{garg2021PUlearning}, 
we present finite sample results
for target label marginal estimation. 
When the data satisfies strong positivity, 
we observe that source classifiers often 
exhibit a threshold $c_y$ on softmax output of 
each class $y \in \out_s$ 
above which the \emph{top bin} (i.e., $[c_y, 1]$)
contains mostly examples from that class $y$. 
We give empirical evidence to this claim in \appref{sec:theorem1_proof}. 
Then, we show that the existence 
of (nearly) pure top bin for each class in $f_s$
is sufficient for Step 3 in \algoref{alg:PULSE} 
to produce (nearly) consistent 
estimates:

\begin{theorem}[Informal] \label{thm:informal_BBE}
Assume that for each class $y\in \out_s$, there exists a threshold $c_y$
such that for the classifier $f_s$, if $[f_s(x)]_y > c_y$ for any $x$ then the true label 
for that sample $x$ is $y$. 
Then, we have $\norm{\wh p_t - p_t}{1} \le \calO\left(\sqrt{\nicefrac{k^3\log(4k/\delta)}{n}} + \sqrt{\nicefrac{k^2\log(4k/\delta)}{m}}\right)  \,$.
\end{theorem}

The proof technique simply builds on the proof of Theorem 1 in \citet{garg2021PUlearning}. 
By assuming that we recover close to ground truth label marginal 
for source classes, we can also extend the above 
analysis to Step 5 of \algoref{alg:PULSE}
to show convergence of estimate $\wh p_t(y=k+1)$ to true prevalence
$p_t(y = k+1)$. We discuss this further in \appref{sec:BBE_ext_app}.  

\paragraph{Theoretical analysis of CVIR in population} 
While the CVIR loss was proposed in \citet{garg2021PUlearning}, 
no analysis was provided
for convergence of the iterative 
gradient descent procedure.
In our work, we show that in population 
on a separable Gaussian dataset,
CVIR will recover the optimal classifier. 

We consider a 
binary classification problem 
where we have access to 
positive distribution (i.e., $p_p$), 
unlabeled distribution 
(i.e., $p_u \defeq \alpha p_p + (1-\alpha) p_n$), 
and mixture coefficient $\alpha$. 
Making a parallel connection to 
Step 5 of PULSE,  
positive distribution $p_p$ here refers
to the label shift corrected source distribution 
$p_s^\prime$ and $p_u$ refers to $p_t = p_t(y=k+1) p_t(x|y=k+1) + (1 - p_t(y=k+1)) p_s^\prime(x)$.
Our goal is to recover the classifier 
that discriminates $p_p$ versus $p_n$ (parallel $p_s^\prime$ versus $p_t(\cdot|y=k+1)$).

First we introduce some notation. For a classifier $f$ and loss function $\ell$ (i.e., logistic loss), define $\vr_\alpha (f) = \inf \{ \tau \in \Real: \Prob_{x\sim p_u}( \ell(x, -1; f) \le \tau) \ge 1 - \alpha \}$. 
Intuitively, $\vr_\alpha(f)$ identifies a threshold $\tau$ to capture bottom $1-\alpha$ fraction of the loss $\ell(x, -1)$ 
for points $x$ sampled from $p_u$.  
Additionally, define CVIR loss as $\calL (f, w) = \alpha \Expt{p_p}{\ell(x,1; f)} + \Expt{p_u}{w(x)\ell(x,-1; f)}$ for classifier $f$ and some weights $w(x) \in\{0,1\}$. 
Formally, given a classifier $f_t$ at an iterate $t$, CVIR procedure proceeds as follows: 
\begin{align}
    w_t(x) &= \indict{\ell(x, -1; f_t) \le \vr_\alpha(f_t)} \,,\label{eq:step1_CVIR}\\ 
    f_{t+1} &= f_t - \eta \grad \calL_f (f_t, w_t) \,. \label{eq:step2_CVIR}
\end{align}

\update{We assume that $x$ are drawn from two half multivariate 
Gaussian with mean zero and identity covariance, i.e., 
    $x \sim p_p \Leftrightarrow x = \gamma_0 \theta_{\opt} + z | \, \theta_\opt^T z \ge 0,$ and  
    $x \sim p_n \Leftrightarrow x = -\gamma_0 \theta_{\opt} + z | \, \theta_\opt^T z < 0,  \text{ where } z \sim \calN(0, I_d)$.
Here $\gamma_0$ is the margin and $\theta_\opt \in \Real^d$ is the true separator. Here, we have access to distribution $p_p$, $p_u = \alpha p_p + (1 - \alpha) p_n$, and the true proportion $\alpha$.} 
%

\update{\begin{theorem}[Informal] \label{thm:informal_CVIR}
In the data setup detailed above, a linear classifier $f(x; \theta) = \sigma\left(\theta^Tx\right)$ trained with CVIR procedure as in \eqref{eq:step1_CVIR}-\eqref{eq:step2_CVIR} will converge to an optimal  positive versus negative classifier. 
\end{theorem}}

\update{The proof uses a key idea that for any classifier $\theta$ not separating positive and negative data perfectly, 
the gradient in \eqref{eq:step2_CVIR} is non-zero. Hence, convergence of the CVIR procedure 
(implied by smoothness of CVIR loss) implies converge to an optimal classifier.} 
For separable datasets in general, we can extend the above analysis with some modifications to the CVIR procedure. We discuss this in \appref{sec:CVIR_ext_app}.  

\paragraph{Empirical investigation in overparameterized models}
As noted in our ablation experiments and in \citet{garg2021PUlearning}, 
domain discriminator trained with CVIR outperforms classifiers 
trained with other consistent objectives (nnPU~\citep{kiryo2017positive} and uPU~\citep{du2015convex}). 
While the above analysis highlights consistency of CVIR procedure
in population, 
it doesn't capture the observed 
empirical efficacy of CVIR over 
alternative methods 
in overparameterized models. 
In the Gaussian setup described 
above, we train overparameterized 
linear models to compare CVIR with other methods.
We discuss precise experiments and results 
in \appref{sec:empirical_CVIR_app}, 
but highlight the key takeaway here. 
First, we observe that when a classifier 
is trained to distinguish positive and unlabeled
data, \emph{early learning}
happens~\citep{liu2020early,arora2019fine,garg2021RATT}, 
i.e., during the initial phase of learning 
classifier learns to classify 
positives in unlabeled correctly as positives.
Next, we show that post early learning  
rejection of large fraction of positives
from unlabeled training  
in equation \eqref{eq:step1_CVIR}
crucially helps CVIR. 



\section{Conclusion} \label{sec:conclusion}
In this work, we introduce OSLS a well-posed instantiation of OSDA
that subsumes label shift and PU learning
into a framework for learning adaptive classifiers.
We presented identifiability conditions 
for OSLS and proposed PULSE, 
a simple and effective approach 
to tackle the OSLS problem. 
Moreover, our extensive experiments 
demonstrate efficacy of PULSE 
over popular OSDA alternatives
when the OSLS assumptions are met. 
We would like to highlight the brittle nature 
of benchmark driven progress in OSDA 
and hope that our work can help to stimulate
more solid foundations and enable 
systematic progress in this area.
Finally, we hope that our open source code and benchmarks
will foster further progress on OSLS. 

\subsection{Limitations and Future Work}

\update{Here, we discuss limitations of the PULSE framework. First, to estimate the relative label shift among source classes in target, we leverage k-PU reductions with several modifications. While we reduce the issues due to overestimation bias by re-normalizing the label marginal among source classes in target, in future, we may hope to replace this heuristic step to directly estimate the joint target marginal.}

\update{Second, since our methods use CVIR and BBE sub-routines, failure of these methods can lead to failure of PULSE. For example, efficacy of BBE relies on the existence of an almost pure top bin in the classifier output space. While this property seems to be satisfied across different datasets spanning different modalities and applications, failure to identify an almost pure top bin can degrade the performance of BBE and hence, our PULSE framework.}

In future work, we also hope to bridge the gap 
between the necessary and sufficient identifiability conditions. 
While we empirically investigate reasons for CVIR's efficacy
in overparameterized models, 
we aim to extend our theory to overparameterized settings in future. 
In our work, we strictly operate under the OSLS settings, where we performed semi-synthetic experiments on vision, language and tabular datasets. In future, it will be interesting to experiment with our PULSE procedure in relaxed settings where $p(x|y)$ may shift in some natural-seeming ways from source to target. 

%



\begin{ack}

We thank Jennifer Hsia for initial discussion on the OSLS problem. We also thank Euxhen Hasanaj 
for suggesting Biology datasets.  
SG acknowledges Amazon Graduate Fellowship for their support. 
SB acknowledges funding from the NSF grants DMS-1713003, DMS-2113684 and CIF-1763734, as well as Amazon AI and a Google Research Scholar Award. 
ZL acknowledges Amazon AI, Salesforce Research, Facebook, UPMC, Abridge, the PwC Center, the Block Center, the Center for Machine Learning and Health, and the CMU Software Engineering Institute (SEI) via Department of Defense contract FA8702-15-D-0002,
for their generous support of ACMI Lab's research on machine learning under distribution shift.


\end{ack}

\bibliographystyle{plainnat}
\bibliography{lifetime}

\newpage 
\section*{Checklist}

\begin{enumerate}

\item For all authors...
\begin{enumerate}
  \item Do the main claims made in the abstract and introduction accurately reflect the paper's contributions and scope?
    \answerYes{}
  \item Did you describe the limitations of your work?
    \answerYes{}
  \item Did you discuss any potential negative societal impacts of your work?
    \answerNA{We believe that this work, which proposes a novel instantiation of open set domain adaptation problem does not present a significant societal concern. 
    While this could potentially guide practitioners
    to improve classification and mixture proportion 
    estimation in applications 
    where data from novel classes can arrive during test time, 
    we do not believe that it will fundamentally 
    impact how machine learning is used in a way 
    that could conceivably be socially salient.}
  \item Have you read the ethics review guidelines and ensured that your paper conforms to them?
    \answerYes{}
\end{enumerate}

\item If you are including theoretical results...
\begin{enumerate}
  \item Did you state the full set of assumptions of all theoretical results?
    \answerYes{See \secref{sec:identifiability} and \secref{sec:analysis}.}
        \item Did you include complete proofs of all theoretical results?
    \answerYes{See \secref{sec:proof_identi_app} and \secref{sec:OSLS_theory_app}.}
\end{enumerate}

\item If you ran experiments...
\begin{enumerate}
  \item Did you include the code, data, and instructions needed to reproduce the main experimental results (either in the supplemental material or as a URL)?
    \answerYes{We include all the experimental details in \appref{sec:exp_app}. We also open source our code at \url{https://github.com/acmi-lab/Open-Set-Label-Shift} }
  \item Did you specify all the training details (e.g., data splits, hyperparameters, how they were chosen)?
    \answerYes{ Yes, see \appref{sec:exp_app}.}
        \item Did you report error bars (e.g., with respect to the random seed after running experiments multiple times)?
    \answerYes{ Yes, we run all experiments with three different seeds and include results in with standard deviation in \appref{sec:exp_app}.}
        \item Did you include the total amount of compute and the type of resources used (e.g., type of GPUs, internal cluster, or cloud provider)?
    \answerYes{Yes, see \appref{sec:exp_app}. }
\end{enumerate}

\item If you are using existing assets (e.g., code, data, models) or curating/releasing new assets...
\begin{enumerate}
  \item If your work uses existing assets, did you cite the creators?
    \answerYes{}
  \item Did you mention the license of the assets?
    \answerYes{}
  \item Did you include any new assets either in the supplemental material or as a URL?
    \answerNA{}
  \item Did you discuss whether and how consent was obtained from people whose data you're using/curating?
    \answerYes{}
  \item Did you discuss whether the data you are using/curating contains personally identifiable information or offensive content?
    \answerNA{}
\end{enumerate}

\item If you used crowdsourcing or conducted research with human subjects...
\begin{enumerate}
  \item Did you include the full text of instructions given to participants and screenshots, if applicable?
    \answerNA{}
  \item Did you describe any potential participant risks, with links to Institutional Review Board (IRB) approvals, if applicable?
    \answerNA{}
  \item Did you include the estimated hourly wage paid to participants and the total amount spent on participant compensation?
    \answerNA{}
\end{enumerate}

\end{enumerate}


\newpage 

\appendix

\section*{\center{Supplementary Materials for Domain Adaptation under Open Set Label Shift}}

\section{Preliminaries} \label{sec:prelims}

\paragraph{Domain adaptation under label shift} Under label shift, we observe data 
from $k$ classes in both source 
and target where the conditional 
distribution remain invariant 
(i.e., $p_s(x|y) = p_t(x|y)$ 
for all classes $y \in [1,k]$) 
but the target label marginal  
may change (i.e., $p_t(y) \ne p_s(y)$). 
Additionally, for all classes 
in source have a non-zero support , i.e., 
for all $y\in [1, k]$, $p_s(y) \ge c$, where $c > 0$. 
Here, given labeled source data and 
unlabeled target data our tasks are: 
(i) estimate the shift in label distribution,
i.e., $p_t(y)$ for all $y\in [1, k]$; 
(ii) train a classifier for the target 
domain $f_t$ to approximate $p_t(y|x)$. 

One common approach to label shift involves estimating the importance ratios $p_t(y)/p_s(y)$
by leveraging a blackbox classifier and then employing re-sampling 
of source data or importance re-weighted ERM on source to obtain a classifier for the target domain~\citep{lipton2018detecting, azizzadenesheli2019regularized, alexandari2019adapting}. 


\paragraph{PU learning}  
Under PU learning, we possess labeled source data 
from a positive class ($p_p$) and unlabeled target data 
from $p_u = \alpha p_p + (1 - \alpha) p_n$ a mixture of 
positive and negative class ($p_n$). 
Our goals naturally break down in to two tasks:  
(i) MPE, determining the fraction of positives $p_p$ in $p_u$
and (ii) PU classification, 
learning a positive-versus-negative classifier on target. 

Note that given access to population of positives and unlabeled, 
$\alpha$ can be estimated as $\min_x p_u(x)/p_p(x)$. 
%
Next, we briefly discuss recent methods for MPE that operate in the classifier 
output space to avoid curse of dimensionality: 

\begin{enumerate}[(i)]
    \item \textbf{EN:} Given a domain discriminator classifier $f_d$ trained to discriminate between positive and unlabeled, \citet{elkan2008learning} proposed the following estimator: ${\sum_{ x_i \in X_p} f_d(x_i) }/{\sum_{ x_i \in X_u} f_d(x_i)}$ where $X_p$ is the set of positive examples and $X_u$ is the set of unlabeled examples. 
    \item \textbf{DEDPUL:} Given a domain discriminator classifier $f_d$, \citet{ivanov2019dedpul} proposed an estimator that leverages density of the data in the output space of the classifier $f_d$ to directly estimate $\min p_u(f(x))/ p_p(f(x))$.  
    
    \item \textbf{BBE:}  BBE~\citep{garg2021PUlearning} identifies a threshold  on probability scores assigned by the classifier $f_d$  such that by estimating the ratio between the fractions of positive and unlabeled  points receiving scores above the threshold, we obtain proportion of positives in unlabeled.
\end{enumerate}

After obtaining an estimate for mixture proportion $\alpha$, following methods can be employed for PU classification: 

\begin{enumerate}[(i)]
    \item \textbf{Domain Discriminator:} Given positive and unlabeled data, \citet{elkan2008learning} trained a classifier $f_d$ to discriminator between them. To make a prediction on test point from unlabeled data, we can then use Bayes rule to obtain the following transformation on probabilistic output of the domain discriminator: $ f = \alpha\left(\frac{m}{n}\right)\left(\frac{f_d(x)}{1 - f_d(x)}\right)$, where $n$ and $m$ are the number of positives and unlabeled examples used to train $f_d$~\citep{elkan2008learning}. 
    
    \item \textbf{uPU:} \citet{du2015convex} proposed an unbiased loss estimator for positive versus negative training. In particular, since $p_u = \alpha p_p + (1-\alpha) p_n$, the loss on negative examples $\Expt{p_n}{\ell(f(x); -1)}$ can be estimated as: 
    \begin{align}
        \Expt{p_n}{\ell(f(x); -1)} = \frac{1}{1-\alpha}\left[\Expt{p_u}{\ell(f(x); -1)} - \alpha \Expt{p_p}{\ell(f(x); -1)}\right]\,.
    \end{align}
    Thus, a classifier can be trained with the following uPU loss: 
    \begin{align}
        \calL_{\text{uPU}} (f) = \alpha \Expt{p_p}{\ell(f(x); +1)} + \Expt{p_u}{\ell(f(x); -1)} - \alpha \Expt{p_p}{\ell(f(x); -1)}\,. \label{eq:uPU_loss}
    \end{align}
    
    \item \textbf{nnPU:} While unbiased losses exist that estimate the PvN loss given PU data  and the mixture proportion $\alpha$, this unbiasedness only holds  before the loss is optimized, and becomes ineffective  with powerful deep learning models  capable of memorization. \citet{kiryo2017positive} proposed the following non-negative regularization for unbiased PU learning: 
    \begin{align}
        \calL_{\text{nnPU}} (f) = \alpha \Expt{p_p}{\ell(f(x); +1)} + \max\left\{\Expt{p_u}{\ell(f(x); -1)} - \alpha \Expt{p_p}{\ell(f(x); -1)}, 0\right\}\,. \label{eq:nnPU_loss}
    \end{align}

    \item \textbf{CVIR: } \citet{garg2021PUlearning} proposed CVIR objective, which discards the highest loss ${\alpha}$ fraction of unlabeled examples on each training epoch, removing the incentive to overfit to the unlabeled positive examples. CVIR loss is defined as 
    \begin{align}
        \calL_{\text{CVIR}} (f) = \alpha \Expt{p_p}{\ell(x,1; f)} + \Expt{p_u}{w(x)\ell(x,-1; f)} \,, 
    \end{align}
    where weights $w(x) = \indict{\ell(x, -1; f) \le \vr_\alpha(f)}$ for $\vr_\alpha (f)$ defined as $\vr_\alpha (f) = \inf \{ \tau \in \Real: \Prob_{x\sim p_u}( \ell(x, -1; f) \le \tau) \ge 1 - \alpha \}$.  Intuitively, $\vr_\alpha(f)$ identifies a threshold $\tau$ to capture bottom $1-\alpha$ fraction of the loss $\ell(x, -1)$  for points $x$ sampled from $p_u$. 
\end{enumerate}



\subsection{Reduction of OSLS into $k$ PU problems } \label{subsec:reduction}

Under the strong positivity condition, 
the OSLS problem can be broken down
into $k$ PU problems as follows: 
By treating a given source class $y_j \in \out_s$
as \emph{positive} and grouping
all other classes together as \emph{negative}
we observe that the unlabeled target data 
is then a mixture of data from 
the positive and negative classes.
This yields a PU learning problem 
and the corresponding mixture proportion
gives the fraction $\alpha_j$ of class $y_j$ among the target data.
By iterating this process for all source classes,
we can solve for the entire target label marginal $\pt(y)$.
Thus, OSLS reduces to $k$ instances of PU learning problem. 
Formally, note that $\pt(x)$ can be written as:
\begin{align}
\pt(x) = \underbrace{\pt(y=j)}_{\alpha_j} \underbrace{\ps(x|y = j)}_{p_p} + \left(1- \pt(y=j)\right) \underbrace{\left(\sum\nolimits_{ i \in \out \setminus \{j\}} \mfrac{\pt(y = i)}{ 1- \pt(y = j)} \ps(x|y = i)\right)}_{p_n} \label{eq:reduction_app}\,, 
\end{align}
individually for all $j \in \out_s$. 
By repeating this reduction for all classes,
we obtain $k$ separate PU learning problems. 
Hence, a natural choice is to leverage 
this structure and solve $k$ PU problems
to solve the original OSLS problem.

In particular, for each class $j \in \out_s$, 
we can first estimate its prevalence $\hat \alpha_j$
in the unlabeled target.  
Then the target marginal for the novel class is given by 
$\smash{\hat \alpha_{k+1} = 1 - \sum_{i=1}^k \hat \alpha_i}$.
For classification,
we can train $k$ PU learning classifiers $f_i$,
where $f_i$ is trained to classify a source class $i$ 
versus others in target. 
Assuming that each $f_j$ returns a score between $[0,1]$, 
during test time, an example $x$ is classified as $f(x)$ 
given by
\begin{align}
    f(x) = \begin{cases}
                \argmax_{ j \in \calY_s } f_j(x) \quad & \text{if } \max_{ j \in \calY_s } f_j(x) \ge 0.5 \\
                k+1 \quad & \text{o.w}\,.
            \end{cases} 
\end{align}
That is, if each classifier classifies the example as 
belonging to other in unlabeled, then 
we classify the example as belonging to the class $k+1$. 
In our main experiments, to estimate $\alpha_j$ and to train $f_j$ classifiers 
for all $j \in \out_s$, we use BBE and CVIR as described before
which was shown to outperform alternative approaches in \citet{garg2021PUlearning}.
We ablate with other methods in \appref{subsec:different_PU_app}. 

\update{Note that mathematically any OSLS problems can be thought of as $k$-PU problems as per \eqref{eq:reduction_app}. However, for identifiablity of each of these PU problems, we need the irreduciblity assumption~\citep{bekker2020learning}. Put simply, for individual PU problems defined for source classes $j \in \out_s$, we need existence of a sub-domain $X_j$ such that we only observe example for that class j in $X_j$. Collectively $X_j$ gives us the $X_\spos$ defined in the strong positivity condition.} 
\update{\paragraph{Failure due to error-accumulation}
While trading off bias with variance, PU learning algorithms tend to over-estimate the mixture proportion~\citep{garg2021PUlearning, bekker2020learning}. This error incurred due to bias can be mild for a single mixture proportion estimation task but accumulates with increasing number of classes (i.e., $k$). This error accumulation can significantly under-estimate the proportion of novel class when estimated by subtracting the sum of prevalence of source classes in target from 1. }

\section{Proofs for identifiability of OSLS} \label{sec:proof_identi_app}

For ease, we re-state \propref{prop:necessary} and \propref{prop:sufficiency}. 

\necessary*
\begin{proof} 
We prove this by contradiction. 
Assume that there exists a unique solution $p_t(y)$. We will obtain contradiction 
when both (i) and (ii) don't hold. 

First, assume for no subset $X_\wpos \subseteq \calX$, 
we have $\left[\ps(x|y)\right]_{x\in X_\wpos, y\in \out_s}$ as full-rank. Then 
in that case, we have vectors $[p_s(x|y = j)]_{x\in \calX}$ as linearly dependent 
for $j \in \out_s$, i.e., there exists $ [\alpha_j]_{j \in \out_s} \in \Real^k$ such that 
$\sum_j \alpha_j p_s(x|y = j) = 0$ for all $x\in \calX$. Thus for small enough $\epsilon > 0$, 
we have infinite solutions of the form $[ p_t(y= j) - \epsilon \cdot a_j]_{j \in \out_s}$.  

Hence, there exists $X_\wpos \subseteq \calX$
for which we have $\left[\ps(x|y)\right]_{x\in X_\wpos, y\in \out_s}$ as full-rank. 
Without loss of generality, we assume that $\abs{X_\wpos} = k$. Assume that $\pt(X_\wpos | y = k+1) > 0$, i.e., 
$[p_t(x|y = k+1)]_{x \in X_\wpos}$ has $l < k$ zero entries. We will 
now construct another solution for the label marginal $p_t$. 
For simplicity we denote $A = \left[\ps(x|y)\right]_{x\in X_\wpos, y\in \out_s}$. Consider
the vector $v (\gamma) = [p_t(x) - (p_t(y = k+1) - \gamma) p_t(x | y = k+1)]_{x \in X_\wpos}$
for some $\gamma > 0$. 
Intuitively, when $\gamma = 0$, we have $ u = A^{-1} v (0)$ where $u = [p_t(y)]_{y \in \out_s}$, i.e., 
we recover the true label marginal corresponding to source classes. 

However, since the solution is not at vertex, there exists a small enough $\gamma > 0$
such that $u ^\prime = A^{-1} v (\gamma)$
with $\sum_j u^\prime_j \le 1$ and $u^\prime_j \ge 0$. Since A is full-rank and $v(\gamma) \ne v(0)$, 
we have $u^\prime \ne u$. 
Thus we construct a separate solution with $u^\prime$ as $[p_t(y)]_{y \in \out_s}$ and 
$p_t(x) - \sum_{j\in \out_s} u^\prime_j p_s(x| y = j)$ as $p_t(x|y = k+1)$. 
Hence, when there exists $X_\wpos \subseteq \calX$
for which we have $\left[\ps(x|y)\right]_{x\in X_\wpos, y\in \out_s}$ as full-rank, 
for uniqueness we obtain a contradiction on the assumption $\pt(X_\wpos | y = k+1) > 0$. 
\end{proof}

We now make some comments on the assumption 
$p_t(y) > 0$ for all $y\in \out_t$ in \propref{prop:necessary}. 
Since, $p_t(y)$ needs to satisfy simplex constraints, if the solution 
is at a vertex of simplex, then OSLS problem may not require 
weak positivity. 
For example, there exists contrived scenarios where 
$p_s(x|y = j) = p_s(x|y = k)$ for all $j, k \in \out_s$ and 
$p_t(x| y= k+1) \ne p_s(x|y =j )$ for all $j \in \out_s$. Then 
when $p_t(x) = p_t(x|y = k+1)$, we can uniquely identify 
the OSLS solution even when weak positivity assumption is not satisfied.

\sufficient*
\begin{proof}
For each condition, we will prove identifiability by constructing the unique solution. 

Under strong positivity, for all $j\in \out_s$ there exists
$x \in X_\spos$ such that $p_t(x|y = k) = 0$ for all $k \in \out_t \setminus \{ j\}$. 
Set $\alpha_j  = \min_{x \in \inpt, p_s(x | y = j) > 0} \frac{p_t(x)}{p_s(x| y = j)}\,,$
for all $j \in \out_s$. For $x \in X_\spos$ such that $p_t(x|y = k) = 0$ 
for all $k \in \out_t \setminus \{ j\}$, we get $\frac{p_t(x)}{p_s(x| y = j)} = p_t(y = j)$ 
and for all $x^\prime \ne x$, we have $\frac{p_t(x)}{p_s(x| y = j)} \ge p_t(y = j)$. 
Thus, we get $\alpha_j = p_t(y = j)$. Finally, we get $\alpha_{k+1} = 1 - \sum_{ j\in \out_s} \alpha_j$.
Plugging in values of the label marginal, we can obtain $p_t(x| y = k+1)$ as 
$p_t(x) - \sum_{y\in \out_s} p_t(y = j) p_s(x| y= j)$. 
%

Under separability, we can obtain the label marginal $p_t$
for source classes by simply considering the set $X_\sep$. Denote 
$A = [p(x|y)]_{x\in X_\sep, y \in \out_s}$ and $v = [p_t(x)]_{x\in X_\sep}$. 
Then, since $A$ is full column-rank by assumption, 
we can define $u = (A^T A)^{-1} A^T v$. For all $x\in X_\sep$, 
we have $p_t(x) = \sum_{y\in \out_s} p_t(y) p_s(x|y)$ and hence, 
$u = [p_t(y)]_{y\in \out_s}$. Having obtained $[p_t(y)]_{y\in \out_s}$, 
we recover $p_t(y = k+1) = 1 - \sum_{ j\in \out_s} p_t(y = j)$ and 
$p_t(x|y = k+1) = p_t(x) - \sum_{ j \in \out_s} p_t(y = j) p_s(x| y=j)$. 
\end{proof}

\subsection{Examples illustrating importance of weak positivity condition} \label{subsec:counter}

In this section, we present two examples, one, to show that weak positivity isn't sufficient 
for identifiability. Second, we present another example where we show that conditions in 
\propref{prop:sufficiency} are not necessary for identifiability.  

\paragraph{Example 1}
Assume $\calX = \{x_1, x_2, x_3, x_4, x_5\}$ and $\calY_t = \{1,2,3\}$. 
Suppose the $p_t(x|y=1)$, $p_t(x|y=2)$, and $p_t(x)$ are given as:
\begin{table}[h]
    \centering
    \begin{tabular}{ c | c | c | c }
       & $p_t(x|y=1)$ & $p_t(x|y=2)$ & $p_t(x)$\\
      \hline
      $x_1$ & $0.4$ & $0.56$ & $0.356$ \\
      \hline
      $x_2$ & $0.3$ & $0.3$ & $0.207$ \\
      \hline
      $x_3$ & $0.2$ & $0.1$ & $0.09$ \\
      \hline
      $x_4$ & $0.1$ & $0.04$ & $0.042$ \\
      \hline
      $x_5$ & $0.0$ & $0.0$ & $0.305$ \\
    \end{tabular}
\end{table}

Here, there exists two separate $p_t(x | y =3)$ and 
$p_t(y)$ that are consistent with the given $p_t(x|y=1)$, 
$p_t(x|y=2)$, and $p_t(x)$ and both the solutions satisfy
weak positivity for two different $X_\wpos$ and $X_\wpos^\prime$.

In particular, notice that $p_t(x| y = 3) = [0.17, 0.0675, 0.0 , 0.0, 0.7625]^T$ 
and $p_t(y) = [0.3, 0.3, 0.4]$ gives us the first solution. 
$p_t(x| y = 3) = [0.0, 0.0, 0.0645, 0.0096, 0.9839]^T$ 
and $p_t(y) = [0.19, 0.5, 0.31]$ gives us another solution. 
For solution 1, $X_\wpos = \{x_3, x_4\}$ and for solution 2, 
$X_\wpos^\prime = \{x_1, x_2\}$.
To check consistency of each solution notice 
that $\sum_{i\in \out} p_t(y = i) p_t(x|y=i) = p_t(x)$ 
for each $x \in \calX$. 
\hfill \qedsymbol{} 

In the above example, the key is to show 
that absent knowledge 
of which $x$'s constitute the set $X_\wpos$, 
we might be able to obtain multiple different 
solutions, each with different $X_\wpos$ 
and both $p_t(y)$, $p_t(x|y=k+1)$ satisfying 
the given information and simplex constraints.

Next, we will show that in certain scenarios
weak positivity is enough for identifiability. 
%

\paragraph{Example 2}
Assume $\calX = \{x_1, x_2, x_3, x_4\}$ and $\calY_t = \{1,2,3\}$. 
Suppose the $p_t(x|y=1)$, $p_t(x|y=2)$, and $p_t(x)$ are given as, 
\begin{table}[h]
    \centering
    \begin{tabular}{ c | c | c | c }
       & $p_t(x|y=1)$ & $p_t(x|y=2)$ & $p_t(x)$\\
      \hline
      $x_1$ & $0.5$ & $0.2$ & $0.24$ \\
      \hline
      $x_2$ & $0.3$ & $0.4$ & $0.2$ \\
      \hline
      $x_3$ & $0.1$ & $0.35$ & $0.35$ \\
      \hline
      $x_4$ & $0.1$ & $0.05$ & $0.21$ \\
    \end{tabular}
\end{table}

Here, out of all $\Mycomb[4]{2}$ possibilities for $X_\wpos$, only 
one possibility yields a solution that satisfies weak positivity and
simplex constraints. 
In particular, the solution is given by $p_t(x| y = 3) = [0.0, 0.0, 0.6, 0.4]^T$
and $p_t(y) = [0.4, 0.2, 0.4]$ with $X_\wpos = \{x_1, x_2\}$. 
\hfill \qedsymbol{} 

In this example, we show that conditions in \propref{prop:sufficiency}
are not necessary to ensure identifiability. 
For discrete domains, this example also highlights 
that we can check identifiability in exponential time
for any OSLS problem 
given $p_t(x)$ and $p_s(x|y)$ for all $y \in \out_s$.

\subsection{Extending identifiability conditions to continuous distributions} \label{subsec:extend_ident_continuos}

To extend our identifiability conditions 
for continuous distributions, 
the linear independence 
conditions on the matrix 
$\left[\ps(x|y)\right]_{x\in X_\sep, y\in \out_s}$
has the undesirable property of being 
sensitive to changes on sets of measure zero.
In particular, by changing a collection 
of linearly dependent distributions 
on a set of measure zero, 
we can make them linearly independent.
As a consequence, we may impose a \emph{stronger} notion 
of independence, i.e., the set of 
distributions $\{p(x|y)\,:\, y=1,...,k\}$ are such that there does
not exist $\vv \ne 0$ for which 
$\int_{X} \lvert{\sum_y p(x|y) v_y}\rvert dx = 0\,,$
where $X = X_\wpos$ for necessary condition 
and $X = X_\spos$ for sufficiency. 
We refer 
this condition as \emph{strict linear independence}. 

\section{PULSE Framework} \label{sec:OSLS_framework_app}

In our PULSE framework, we build on top of BBE and CVIR from \citet{garg2021PUlearning}. 
\update{Here, we elaborate on Step 3 and 5 in \algoref{alg:PULSE}.}  
\update{\paragraph{Extending BBE algorithm to estimate target marginal among previously seen classes} 
We first explain the intuition behind BBE approach.
In a PU learning problem, 
given positive and unlabeled data, 
BBE estimates the fraction of positives 
in unlabeled in the push-forward space of the classifier.
In particular, instead of operating in the original input space, BBE maps the inputs to one-dimensional outputs (i.e., a score between zero and one) which is the predicted probability of an example being from the positive class. 
BBE identifies a threshold on probability scores assigned 
by a domain discriminator classifier
such that the ratio between
the fractions of positive and unlabeled 
points receiving scores above the threshold is minimized. 
Intuitively, if their exists a threshold 
on probability scores assigned by the classifier
such that the examples mapped to a score greater than the threshold are \emph{mostly} positive, BBE aims to identify this threshold.  
Efficacy of BBE procedure relies on existence of such a threshold.  
This is referred to as the \emph{top bin property}. We provide empirical evidence to the property in \figref{fig:loss_bin_property} in \appref{sec:theorem1_proof}.
We tailor BBE to estimate the relative fraction 
of previously seen classes in the target distribution
by exploiting a $k$-way source classifier $f_s$ 
trained on labeled source data. 
We describe the procedure in \algoref{alg:BBE}. 
}

\update{We now introduce some notation needed to introduce the tailored BBE proceudre formally.}
For given probability density function $p$ 
and a scalar output function $f$, 
define a function $q(z) = \int_{ A_z} p(x) dx$, 
where $A_z = \{x \in \inpt: f(x) \ge z\}$ for all $z\in [0,1]$. 
Intuitively, $q(z)$ captures the 
cumulative density of points in a top bin, 
the proportion of input domain 
that is assigned a value larger than $z$ 
by the function $f$ in the transformed space. 
We define an empirical estimator $\wh q(z)$
given a set $X = \{x_1, x_2, \ldots, x_n\}$
sampled iid from $p(x)$. Let $Z = f(X)$. 
Define $ \wh q(z) = \sum_{i=1}^n \indict{z_i \ge z}/{n}$.

\update{Our modified BBE procedure proceeds as follows.} 
Given a held-out dataset of source 
$\{{\bf X}^S_2, {\bf y}^S_2\}$ and unlabeled target samples ${\bf X}^T_2$,
we push all examples through the source classifier $f$
to obtain $k$ dimensional outputs. 
For all $j \in \out_s$, we repeat the
following:  Obtain $Z_s = f_j({\bf X}^S_2 [\text{id}_j])$ 
and $Z_t = f_j({\bf X}^T_2)$. 
\update{Intuitively, $Z_s$ and $Z_t$ are the push forward mapping of the source classifier.}
Next, with $Z_p$ and $Z_u$, 
we estimate $\wh q_s$ and $\wh q_t$. 
Finally, we estimate $[\wh p_t]_j$ as the ratio 
$\wh q_t (\wh c) / \wh q_s (\wh c)$ at $\wh c$ 
that minimizes the upper confidence bound
at a pre-specified level $\delta$ 
and a fixed parameter $\gamma \in (0,1)$.  
Our method is summarized in \algoref{alg:BBE}. 
Throughout all the experiments, we fix $\delta$ at $0.1$
and $\gamma$ at $0.01$.

\begin{algorithm}[h]
  \caption{Extending Best Bin Estimation (BBE) for Step 3 in \algoref{alg:PULSE}}
  \label{alg:BBE}
  \begin{algorithmic}[1]
  \INPUT: Validation source  $\{{\bf X}^S_2, {\bf y}^S_2\}$ and unlabeled target samples ${\bf X}^T_2$.
   Source classifier ${f}: \calX \to \Delta^{k-1}$.  Hyperparameter $0< \delta,\gamma <1$.
    \STATE $ \wh p_t \gets \textrm{zeros}(size=\abs{\out_s}) $
    \FOR{ $j \in \out_s$}
        \STATE $\text{id}_j \gets \text{where}({\bf y}^S_2 = j)$.
        \STATE $Z_s, Z_t \gets \left[f({\bf X}^S_2[\text{id}_j])\right]_j, \left[f({\bf X}_2^T)\right]_j$. 
        \STATE $\wh q_s (z), \wh q_t(z) \gets \frac{\sum_{z_i \in  Z_s} \indict{z_i \ge z}}{\abs{\text{id}_j} }, \frac{\sum_{z_i \in  Z_t} \indict{z_i \ge z}}{ \abs{{\bf X}_2^T} }$ for all $z \in [0,1]$. 
        \STATE $\wh c_j \gets \argmin_{c \in [0,1]} \left( \frac{\wh q_t(c)}{\wh q_s(c)}  + \frac{1 + \gamma}{\wh q_s(c)}\left( \sqrt{\frac{\log(4/\delta)}{2 \abs{{\bf X}_2^T}}} + \sqrt{\frac{\log(4/\delta)}{2\abs{\text{id}_j}}}\right) \right)\,$. 
        \STATE $ [\wh p_t]_j \gets \frac{\wh q_t(\wh c_j)}{\wh q_s(\wh c_j)}$.
    \ENDFOR
    \OUTPUT: Normalized target marginal among source classes $\wh p^\prime_t \gets \frac{\wh p_t}{\norm{\wh p_t}{1}}$
\end{algorithmic}
\end{algorithm}

\paragraph{Extending CVIR to train discriminator $f_d$ and estimate novel class prevalence} 
After estimating the fraction 
of source classes in target (i.e., $p_t^\prime(j) =\nicefrac{p_t(y=j)}{\sum_{k\in \out_s} p_t(y=k)}$ for all $j \in \out_s$), 
we re-sample the source data according to $p_t^\prime(y)$
to mimic samples from distribution $p_s^\prime(x)$. 
Thus, obtaining a PU learning problem instance, 
we resort to PU learning techniques to 
(i) estimate the fraction of novel class $p_t(y=k+1)$; 
and (ii) learn a binary classifier $f_d(x)$ to discriminate 
between label shift corrected source $p_s^\prime(x)$
and novel class $p_t(x|y=k+1)$. 
Assume that sigmoid output $f_d(x)$ 
indicates predicted probability of an example 
$x$ belonging to label shift 
corrected source $p_s^\prime(x)$.
\update{With $\wh \calL^+(f_\theta; X)$, we denote the loss incurred by $f_\theta$ when classifying examples from $X$ as positive, i.e.,  
$\wh \calL^+(f_\theta; X)  = \sum_{i = 1}^{\abs{X}} \frac{\ell( f_\theta(x_i), +1)}{\abs{X}}$. Similarly, $\wh \calL^-(f_\theta; X)  = \sum_{i = 1}^{\abs{X}} \frac{\ell( f_\theta(x_i), -1)}{\abs{X}}$} 

\update{Given an estimate of the 
fraction of novel class $\wh p_t(y=k+1)$,
CVIR objective creates a provisional set 
of novel examples ${\bf X}^N_1$
by removing 
$(1 - \wh p_t(y=k+1))$ fraction of examples
from ${\bf X}^T_1$
that incur highest loss when predicted as novel class 
on each training epoch.} 
Next, we update our discriminator $f_d$
by minimizing loss on label shift corrected source 
$\wt {\bf X}^S_1$ and provisional novel examples ${\bf X}^N_1$.
This step is aimed to remove any  
incentive to overfit to
the examples from $p_s^\prime(x)$. 
Consequently, we employ the iterative procedure
that alternates between 
estimating the prevalence 
of novel class $\wh p_t(y=k+1)$ (with BBE) and 
minimizing the CVIR loss with estimated 
fraction of novel class. 
\algoref{alg:CVIR} summarizes our 
approach which is used in Step 3 of \algoref{alg:PULSE}. 

Note that we need 
to warm start with simple domain 
discrimination training, 
since in the 
initial stages mixture proportion
estimate is often close to 1 
rejecting all the 
unlabeled examples.
In \citet{garg2021PUlearning}, it was shown
that the procedure is not sensitive 
to the choice of number of warm start epochs
and in a few cases with large datasets, 
we can even get away without warm start 
(i.e., $W=0$) without hurting 
the performance. In our work, we notice 
that given an estimate $\hat \alpha$ of 
prevalence of novel class, we can use
unbiased PU error \eqref{eq:uPU_loss} on validation data 
as a surrogate to identify warm start epochs for domain discriminator training. 
In particular, we train the domain discriminator 
classifier for a large number of epochs, say $E (>> W)$,
and then choose the discriminator, 
i.e., warm start epoch $W$ at which 
$f_d$ achieves minimum unbiased validation loss.

Finally, to obtain a $(k+1)$-way classifier $f_t(x)$ on target 
we combine discriminator $f_d$ and  source classifier $f_s$
with importance-reweighted label shift correction. 
In particular,
for all $j \in \out_s$, $[f_t(x)]_j = (f_d(x)) \frac{w(j) \cdot [f_s(x)]_j}{\sum_{k\in \out_s} w(k) \cdot [f_s(x)]_k} $
and $[f_t(x)]_{k+1} = 1 - f_d(x)$. 
Similarly, to obtain 
target marginal $p_t$, 
we re-scale the label shift estimate among previously 
seen classes with estimate of prevalence of novel examples, i.e., 
for all $j \in \out_s$, assign $\wh p_t(y = j) = (1-\wh p_t(y=k+1))\cdot \wh p_t^\prime(y = j)$. 

Overall, our approach proceeds
as follows (\algoref{alg:PULSE}):
First, we estimate the label shift among
previously seen classes. 
Then we employ importance re-weighting of source 
data to formulate a single PU learning problem 
between source and target to estimate fraction 
of novel class $\wh p_t(y=k+1)$ and 
to learn a discriminator $f_d$ 
for the novel class. Combining discriminator 
and label shift corrected source classifier
we get $(k+1)$-way target classifier.

\begin{algorithm}[h]
  \caption{Alternating between CVIR and BBE for Step 5 in \algoref{alg:PULSE} }
  \label{alg:CVIR}
  \begin{algorithmic}[1]
  \INPUT: Re-sampled training source data $\widetilde{\bf X}^S_1$, validation source data $ \widetilde{\bf X}^S_2$. Training target data ${\bf X}^T_1$ and validation data ${\bf X}^T_2$. Hyperparameter $W, B, \delta, \gamma$. 
    \STATE Initialize a training model $f_\theta$ and an stochastic optimization algorithm $\calA$.
    \STATE ${\bf X}^N_1 \gets {\bf X}^T_1$.

    \COMMENT {// Warm start with domain discrimination training}
    \FOR{ $i \gets 1$ to $W$} 
        \STATE{Shuffle $(\widetilde{\bf X}^S_1, {\bf X}^N_1)$ into $B$ mini-batches. With $(\widetilde{\bf X}^S_1[i], {\bf X}^N_1[i])$ we denote $i^\text{th}$ mini-batch}. 
        \FOR{ $i \gets 1$ to $B$} 
            \STATE Set the gradient $\grad_\theta \left[  \wh \calL^+(f_\theta; \widetilde{\bf X}^S_1[i]) + \wh \calL^-(f_\theta; {\bf X}^N_1[i]) \right]$ and update $\theta$ with algorithm $\calA$.
        \ENDFOR
    \ENDFOR
    \STATE $\wh \alpha \gets $  BBE($\widetilde{\bf X}^S_2, {\bf X}^T_2, f_\theta$) \hfill \COMMENT{\algoref{alg:MPE_PU}}
    \STATE Rank samples $x \in {\bf X}^T_1 $ according to their loss values $\lossell( f_\theta(x),-1)$.
    \STATE ${\bf X}^N_1 \gets \{{\bf X}^T_1\}_{1-\wh \alpha}$ where $\{{\bf X}^T_1\}_{1-
    \wh \alpha}$ denote the lowest ranked $1- \wh \alpha$ fraction of samples. 
    \WHILE{training error $\wh \calE^+(f_\theta; \widetilde{\bf X}^S_2 ) + \wh\calE^-(f_\theta; {\bf X}^N_1)$ is not converged}
        \STATE Train model $f_\theta$ for one epoch on  $(\widetilde{\bf X}^S_1, {\bf X}^N_1)$  as in Lines 4-7.
        \STATE $\wh \alpha \gets $ BBE($\widetilde{\bf X}^S_2, {\bf X}^T_2, f_\theta$) \hfill \COMMENT{\algoref{alg:MPE_PU}}
        \STATE Rank samples $x \in {\bf X}^T_1 $ according to their loss values $\lossell( f_\theta(x),-1)$.
        \STATE ${\bf X}^N_1 \gets \{{\bf X}^T_1\}_{1-\wh \alpha}$ where $\{{\bf X}^T_1\}_{1-\wh\alpha}$ denote the lowest ranked $1- \wh \alpha$ fraction of samples. 
    \ENDWHILE
  \OUTPUT: Trained discriminator $f_d \gets f_\theta$ and novel class fraction  $\wh p_t(y=k+1) \gets 1 - \wh \alpha$. 
\end{algorithmic}
\end{algorithm}

\begin{algorithm}[h]
  \caption{Best Bin Estimation (BBE)}
  \label{alg:MPE_PU}
  \begin{algorithmic}[1]
  \INPUT: Re-sampled source data $\widetilde{\bf X}^S$ and target samples ${\bf X}^T$. Discriminator classifier $\hat{\f}: \calX \to [0,1]$. Hyperparameter $0< \delta,\gamma <1$.
    \STATE $Z_s, Z_t \gets f(\widetilde{\bf X}^S), f({\bf X}^T)$. 
    \STATE $\wh q_t (z), \wh q_s(z) \gets \frac{\sum_{z_i \in  Z_s} \indict{z_i \ge z}}{\abs{\widetilde{\bf X}^S}}, \frac{\sum_{z_i \in  Z_t} \indict{z_i \ge z}}{\abs{\bf X}^T}$ for all $z \in [0,1]$. 
    \STATE  Estimate $\wh c \gets \argmin_{c \in [0,1]} \left( \frac{\wh q_t(c)}{\wh q_s(c)}  + \frac{1 + \gamma}{\wh q_s(c)}\left( \sqrt{\frac{\log(4/\delta)}{2 \abs{\widetilde{\bf X}^S}}} + \sqrt{\frac{\log(4/\delta)}{2\abs{{\bf X}^T}}}\right) \right)\,$. 
    \OUTPUT:  $\wh \alpha \gets \frac{\wh q_t(\wh c)}{\wh q_s(\wh c)}$ 
\end{algorithmic}
\end{algorithm}

\subsection{PULSE under separability} \label{subsec:pulse_separablity_app}

Our ideas for PULSE framework 
can be
extended to separability 
condition since   
\eqref{eq:label_shift_corr}
continues to hold. 
In particular, when OSLS 
satisfies the separability assumption, 
we may hope to jointly estimate the label 
shift among previously seen classes 
with label shift estimation techniques~\citep{lipton2018detecting, alexandari2019adapting}
and learn a domain discriminator classifier. 
This may be achieved by estimating label shift 
among examples rejected by 
domain discriminator classifier as belonging to 
previously seen classes. 
However, in our initial experiments,
we observe that techniques proposed 
under strong positivity were empirically 
stable and outperform methods
developed under separability.
This is intuitive for many benchmark datasets 
where it may be more natural to expect that 
for each class there exists a 
subdomain that only belongs to that class
than assuming separability only between 
novel class samples
and examples from source classes. 
%

\section{Proofs for analysis of OSLS framework} \label{sec:OSLS_theory_app}

In this section, we provide missing formal statements
and proofs for theorems in \secref{sec:analysis}. 
This mainly includes analysing
key steps of our PULSE procedure
for target label marginal estimation 
(Step 3, 5 \algoref{alg:PULSE})
and learning the domain discriminator 
classifier (Step 5, \algoref{alg:PULSE}). 

\subsection{Formal statement and proof of Theorem 1}\label{sec:theorem1_proof}

Before introducing the formal statement, 
we introduce some additional notation. 
Given probability density function $p$ 
and a source classifier $f : \calX \to \Delta^{k-1}$,
define a function $q(z, j) = \int_{ A(z,j)} p(x) dx$, 
where $A(z, j) = \{x \in \inpt: [f(x)]_j \ge z\}$ for all $z\in [0,1]$. 
Intuitively, $q(z, j)$ captures the 
cumulative density of points in a top bin 
for class $j$, i.e., 
the proportion of input domain 
that is assigned a value larger than $z$ 
by the function $f$ at the index $j$ 
in the transformed space. 
We define an empirical estimator $\wh q(z, j)$
given a set $X = \{x_1, x_2, \ldots, x_n\}$
sampled iid from $p(x)$. Let $Z = [f(X)]_j$. 
Define $ \wh q(z, j) = \sum_{i=1}^n \indict{z_i \ge z}/{n}$.

For each pdf $p_s$ and $p_t$, 
we define $q_s$ and $q_t$ respectively. 
Moreover, for each class $j \in \out_s$, 
we define $q_{t,j}$ corresponding to 
$p_{t, j}\defeq p_t(x| y = j)$ and 
$q_{t, -j}$  corresponding to 
$p_{t, -j}\defeq 
\frac{ \sum_{ i \in \out_t \setminus \{j\} } p_t(y = i) p_t(x| y = i)}{\sum_{ i \in \out_t \setminus \{j\} } p_t(y= j)}$. Assume that we have 
$n$ source examples and $m$ target examples. 
%
Now building on BBE results from \citet{garg2021PUlearning}, 
we present finite sample results
for target label marginal estimation:

\begin{theorem}[Formal statement of \thmref{thm:informal_BBE}] \label{thm:formal_BBE}
Define $c^*_j = \argmin_{c \in [0,1]} \left( {q_{t, -j}(c, j)}/{q_{t, j}(c, j)} \right)$, 
for all $j \in \out_s$.
Assume $\min(n, m) \ge \max_{j\in \out_s} \left( \frac{2\log(4k/\delta)}{q_{t,j}^2(c^*_j, j)} \right)$.
Then, for every $\delta > 0$, $\wh p_t$ (in~\algoref{alg:BBE} with $\delta$ as $\delta/k$) satisfies with probability 
at least $1-\delta$, we have: 
\begin{align*}
    \norm{\wh p_t - p_t}{1} \le \sum_{j \in \out_s} \left(1 - p_t(y = j)\right) \left( \frac{q_{t, -j}(c^*_j, j)}{q_{t, j}(c^*_j, j)} \right) + \calO\left(\sqrt{\frac{k^3\log(4k/\delta)}{n}} + \sqrt{\frac{k^2\log(4k/\delta)}{m}}\right)  \,.
\end{align*}
%
\end{theorem}

When the data satisfies strong positivity, 
we observe that source classifiers often 
exhibit a threshold $c_y$ on softmax output of 
each class $y \in \out_s$ 
above which the \emph{top bin} (i.e., $[c_y, 1]$)
contains mostly examples from that class $y$. 
Formally, as long as 
there exist a threshold $c^*_j \in (0,1)$ 
such that $q_{t, j}(c^*_j) \ge \epsilon$ and $ q_{t, -j}(c^*_j) = 0$ 
for some constant $\epsilon >0$ for all $j\in \out_s$, 
we show that our estimator $\wh \alpha$ converges 
to the true $\alpha$ with convergence rate $\min(n,m)^{-1/2}$. 
The proof technique simply builds on 
the proof of Theorem 1 in \citet{garg2021PUlearning}. 
First, we state Lemma 1 from \citet{garg2021PUlearning}. 
Next, for completeness we provide the proof for \thmref{thm:formal_BBE}
which extends proof of Theorem 1~\citep{garg2021PUlearning}
for $k$ classes. 

\begin{lemma}\label{lem:ucb}
Assume two distributions $q_p$ and $q_u$ with their empirical 
estimators denoted by $\wh q_p$ and $\wh q_u$ respectively. 
Then for every $\delta >0$, with probability at least $1-\delta$, 
we have for all $c \in [0,1]$
\begin{align*}
    \abs{ \frac{\wh q_u(c)}{ \wh q_p(c)} -  \frac{q_u(c)}{q_p(c)}} \le \frac{1}{\wh q_p(c)}\left( \sqrt{\frac{\log(4/\delta)}{2n_u}} + \frac{q_u(c)}{q_p(c)}\sqrt{\frac{\log(4/\delta)}{2n_p}}\right) \,.
\end{align*}
\end{lemma}  

\begin{proof}[Proof of \thmref{thm:formal_BBE}] 
The main idea of the proof is to use the confidence bound derived in \lemref{lem:ucb} at $\wh c$ and use the fact that $\wh c$ minimizes the upper confidence bound. The proof is split into two parts. First, we derive a lower bound on $\wh q_{t, j}(\wh c_j)$ for all $j \in \out_s$ and next, we use the obtained lower bound to derive confidence bound on $\wh p_t(y = j)$. 
With $\wh \alpha_j$, we denote $\wh p_t(y = j)$ for all $j\in \out_s$. All the statements in the proof simultaneously hold with probability $1-\delta/k$. We derive the bounds for a single $j\in \out_s$ and then use union bound to combine bound for all $j \in \out_s$. When it is clearly from context, we denote $q_{t, j} (c ,j)$ with $q_{t,j}(c)$ and $q_{t} (c ,j)$ with $q_{t}(c)$. 
Recall,
    \begin{align}
        \wh c_j \defeq \argmin_{c \in [0,1]} & \frac{\wh q_t(c)}{\wh q_{t, j}(c)}  + \frac{1}{\wh q_{t,j}(c)}\left( \sqrt{\frac{\log(4k/\delta)}{2 m}} + (1+\gamma)\sqrt{\frac{\log(4k/\delta)}{2n p_s(y = j)}}\right) \qquad \text{and} \\
        \wh p_t(y = j) &\defeq \frac{\wh q_t(\wh c_j)}{\wh q_{t, j}(\wh c_j)}\,.
    \end{align}
    Moreover, 
    \begin{align}
        c^*_j \defeq \argmin_{c \in [0,1]} \frac{q_t(c)}{q_{t, j}(c)} \qquad \text{and}\qquad \alpha^*_j \defeq \frac{q_t(c^*_j)}{q_{t, j}(c^*_j)}\,.
    \end{align}
    \textbf{Part 1:}  We establish lower bound on $\wh q_{t, j}(\wh c_j)$. Consider $c^\prime_j \in [0,1]$ such that $\wh q_{t,j}(c^\prime_j) = \frac{\gamma}{2 + \gamma} \wh q_{t,j}(c^*_j)$. We will now show that \algoref{alg:BBE} will select $\wh c_j < c^\prime_j$. For any $c \in [0,1]$, we have with with probability $1-\delta/k$,  
    \begin{align}
        \wh q_{t,j}(c) - \sqrt{\frac{\log(4k/\delta)}{2n \cdot p_s(y = j)}} \le q_{t,j}(c) \qquad \text{and} \qquad q_t(c) - \sqrt{\frac{\log(4k/\delta)}{2m}} \le \wh q_t(c) \,.
    \end{align}
    Since $ \frac{q_t(c^*_j)}{q_{t,j}(c^*_j)} \le \frac{q_t(c)}{q_{t,j}(c)}$, we have 
    \begin{align}
        \wh q_t(c) \ge q_{t,j}(c) \frac{q_t(c^*_j)}{q_{t,j}(c^*_j)}  - \sqrt{\frac{\log(4k/\delta)}{2m}} \ge \left(  \wh q_{t,j}(c) - \sqrt{\frac{\log(4k/\delta)}{2n \cdot p_s(y=j)}}  \right) \frac{q_t(c^*_j)}{q_{t,j}(c^*_j)}  - \sqrt{\frac{\log(4k/\delta)}{2 m}} \,.
    \end{align}
    Therefore, at $c$ we have 
    \begin{align}
        \frac{\wh q_t(c)}{\wh q_{t,j}(c)}  &\ge \alpha^*_j -  \frac{1}{\wh q_{t,j}(c)}\left( \sqrt{\frac{\log(4k/\delta)}{2m}} + \frac{q_t(c^*_j)}{q_p(c^*_j)}\sqrt{\frac{\log(4k/\delta)}{2n \cdot p_s(y = j)}}\right) \,.
    \end{align}

    Using \lemref{lem:ucb} at $c^*$, we have 
    \begin{align}
        \frac{\wh q_t(c)}{\wh q_{t,j}(c)}  &\ge \frac{\wh q_t(c^*_j)}{\wh q_{t,j}(c^*_j)} - \left(\frac{1}{\wh q_{t,j}(c^*_j)} +  \frac{1}{\wh q_{t,j}(c)}\right)\left( \sqrt{\frac{\log(4k/\delta)}{2m}} + \frac{q_t(c^*_j)}{q_{t,j}(c^*_j)}\sqrt{\frac{\log(4k/\delta)}{2n \cdot p_s(y = j) }}\right) \\
        &\ge \frac{\wh q_t(c^*_j)}{\wh q_{t,j}(c^*_j)} - \left(\frac{1}{\wh q_{t,j}(c^*_j)} +  \frac{1}{\wh q_{t,j}(c)}\right)\left( \sqrt{\frac{\log(4k/\delta)}{2m}} + \sqrt{\frac{\log(4k/\delta)}{2n \cdot p_s(y = j)}}\right) \,,
    \end{align}
    where the last inequality follows from the fact that $\alpha^*_j = \frac{q_t(c^*_j)}{q_{t,j}(c^*_j)} \le 1$. Furthermore, the upper confidence bound at $c$ is lower bound as follows: 
    \begin{align}
        \frac{\wh q_t(c)}{\wh q_{t,j}(c)} + &\frac{1+\gamma}{\wh q_{t,j}(c)} \left( \sqrt{\frac{\log(4l/\delta)}{2m}} + \sqrt{\frac{\log(4k/\delta)}{2n\cdot p_s(y = j)}}\right) \\ &\ge \frac{\wh q_t(c^*_j)}{\wh q_{t,j}(c^*_j)} + \left(\frac{1 + \gamma}{\wh q_{t,j}(c)} - \frac{1}{\wh q_{t,j}(c^*_j)} - \frac{1}{\wh q_{t,j}(c)}\right)\left( \sqrt{\frac{\log(4k/\delta)}{2m}} + \sqrt{\frac{\log(4k/\delta)}{2n \cdot p_s(y = j)}}\right) \\  
        &= \frac{\wh q_t(c^*_j)}{\wh q_{t,j}(c^*_j)} +  \left(\frac{\gamma}{\wh q_{t,j}(c)} - \frac{1}{\wh q_{t,j}(c^*_j)} \right)\left( \sqrt{\frac{\log(4k/\delta)}{2m}} + \sqrt{\frac{\log(4k/\delta)}{2n \cdot p_s(y = j)}}\right) \label{eq:lower_ucb}
    \end{align}
    Using \eqref{eq:lower_ucb} at $c = c^\prime$, we have the following lower bound on ucb at $c^\prime$: 
    \begin{align}
        \frac{\wh q_t(c^\prime)}{\wh q_{t,j}(c^\prime)} + &\frac{1+\gamma}{\wh q_{t,j}(c^\prime)} \left( \sqrt{\frac{\log(4k/\delta)}{2m}} + \sqrt{\frac{\log(4k/\delta)}{2n \cdot p_s(y = j)}}\right) \\ 
        &\ge \frac{\wh q_t(c^*_j)}{\wh q_{t,j}(c^*_j)} + \frac{1 + \gamma}{\wh q_{t,j}(c^*_j)}\left( \sqrt{\frac{\log(4k/\delta)}{2m}} + \sqrt{\frac{\log(4k/\delta)}{2n \cdot p_s(y = j)}}\right) \,,
    \end{align}
    
    Moreover from \eqref{eq:lower_ucb}, we also have that the lower bound on ucb at $c \ge c^\prime$ is strictly greater than the lower bound on ucb at $c^\prime$. Using definition of $\wh c$, we have 
    \begin{align}
         \frac{\wh q_t(c^*_j)}{\wh q_{t,j}(c^*_j)} &+ \frac{1 + \gamma}{\wh q_{t,j}(c^*_j)}\left( \sqrt{\frac{\log(4k/\delta)}{2m}} + \sqrt{\frac{\log(4k/\delta)}{2n \cdot p_s(y = j)}}\right) \\ 
        &\ge \frac{\wh q_t(\wh c)}{\wh q_{t,j}(\wh c)} + \frac{1 + \gamma}{\wh q_{t,j}(\wh c)}\left( \sqrt{\frac{\log(4k/\delta)}{2m}} + \sqrt{\frac{\log(4k/\delta)}{2n \cdot p_s(y = j)}}\right) \,,
    \end{align}
    and hence 
    \begin{align}
        \wh c \le c^\prime \,.
    \end{align}

    \textbf{Part 2:} We now establish an upper and lower bound on $\wh \alpha_j$. We start with upper confidence bound on $\wh \alpha_j$. By definition of $\wh c_j$, we have
    
    \begin{align}
    \frac{\wh q_t(\wh c)}{\wh q_{t,j}(\wh c)} + \frac{1 + \gamma}{\wh q_{t,j}(\wh c)} & \left( \sqrt{\frac{\log(4k/\delta)}{2m}} + \sqrt{\frac{\log(4k/\delta)}{2n \cdot p_s(y = j)}}\right)   \\
    & \le \min_{c \in [0,1]} \left[ \frac{\wh q_t(c)}{\wh q_{t,j}(c)} + \frac{1 + \gamma}{\wh q_{t,j}(c)}\left( \sqrt{\frac{\log(4k/\delta)}{2m}} + \sqrt{\frac{\log(4k/\delta)}{2n \cdot p_s(y = j)}}\right) \right] \\ 
    & \le \, \frac{\wh q_t(c^*_j)}{\wh q_{t,j}(c^*_j)}  + \frac{1+\gamma}{\wh q_{t,j}(c^*_j)}\left( \sqrt{\frac{\log(4k/\delta)}{2m}} + \sqrt{\frac{\log(4k/\delta)}{2n \cdot p_s(y = j)}}\right)  
    \,. \label{eq:bound_step1} \numberthis 
    \end{align}
    Using \lemref{lem:ucb} at $c^*_j$, we get 
    \begin{align*}
        \frac{\wh q_t(c^*_j)}{\wh q_{t,j}(c^*_j)} &\le \frac{q_t(c^*_j)}{q_{t,j}(c^*_j)} + \frac{1}{\wh q_{t,j}(c^*_j)}\left( \sqrt{\frac{\log(4k/\delta)}{2m}} + \frac{q_t(c^*_j)}{q_{t,j}(c^*_j)}\sqrt{\frac{\log(4k/\delta)}{2n \cdot p_s(y = j)}}\right) \\
        &= \alpha_j^* + \frac{1}{\wh q_{t,j}(c^*_j)}\left( \sqrt{\frac{\log(4k/\delta)}{2m}} + \alpha_j^*\sqrt{\frac{\log(4k/\delta)}{2n \cdot p_s(y = j)}}\right) \,.\numberthis \label{eq:bound_step2}
    \end{align*} 
    Combining \eqref{eq:bound_step1} and \eqref{eq:bound_step2}, we get 
    \begin{align}
        \wh \alpha_j = \frac{\wh q_t(\wh c)}{\wh q_{t,j}(\wh c)} \le \alpha_j^* + \frac{2+\gamma}{\wh q_{t,j}(c^*_j)}\left( \sqrt{\frac{\log(4k/\delta)}{2m}} + \sqrt{\frac{\log(4k/\delta)}{2n \cdot p_s(y = j)}}\right)  \,.
    \end{align}
    
    Using DKW inequality on $\wh q_{t,j}(c^*_j)$, we have $\wh q_{t,j}(c^*_j) \ge q_{t,j}(c^*_j) - \sqrt{\frac{\log(4k/\delta)}{2n \cdot p_s(y = j)}}$. Assuming $n \cdot p_s(y = j) \ge \frac{2\log(4k/\delta)}{q_{t,j}^2(c^*_j)}$, we get $\wh q_{t,j}(c^*_j) \le q_{t,j}(c^*_j)/ 2$ and hence, 
    \begin{align}
        \wh \alpha_j  \le \alpha_j^* + \frac{4+2\gamma}{q_{t,j}(c^*_j)}\left( \sqrt{\frac{\log(4k/\delta)}{2m}} + \sqrt{\frac{\log(4k/\delta)}{2n \cdot p_s(y = j)}}\right)  \,. \label{eq:upper_bound}
    \end{align}
    
    Finally, we now derive a lower bound on $\wh \alpha_j$. From \lemref{lem:ucb}, we have the following inequality at $\wh c$ 
    \begin{align}
          \frac{q_t(\wh c)}{q_{t,j}( \wh c)} \le \frac{\wh q_t(\wh c)}{\wh q_{t,j}(\wh c)} + \frac{1}{\wh q_{t,j}(\wh c)}\left( \sqrt{\frac{\log(4k/\delta)}{2m}} + \frac{q_t(\wh c)}{q_{t,j}(\wh c)}\sqrt{\frac{\log(4k/\delta)}{2n \cdot p_s(y = j)}}\right) \,. \label{eq:lem1}
    \end{align}
    Since $\alpha_j^* \le \frac{q_t(\wh c)}{q_{t,j}( \wh c)} $, we have 
    \begin{align}
        \alpha_j^* \le \frac{q_t(\wh c)}{q_{t,j}( \wh c)} \le \frac{\wh q_t(\wh c)}{\wh q_{t,j}(\wh c)} + \frac{1}{\wh q_{t,j}(\wh c)}\left( \sqrt{\frac{\log(4k/\delta)}{2m}} + \frac{q_t(\wh c)}{q_{t,j}(\wh c)}\sqrt{\frac{\log(4k/\delta)}{2n \cdot p_s(y = j)}}\right) \,. \label{eq:lower_bound_step1}
    \end{align}
    
    Using \eqref{eq:upper_bound}, we obtain a very loose upper bound on $\frac{\wh q_t(\wh c)}{\wh q_{t,j}(\wh c)}$. Assuming $\min(n \cdot p_s(y = j), m) \ge \frac{2\log(4k/\delta)}{q_{t,j}^2(c^*_j)}$, we have $\frac{\wh q_t(\wh c)}{\wh q_{t,j}(\wh c)} \le \alpha_j^* + 4 + 2\gamma \le 5 + 2\gamma$. Using this in \eqref{eq:lower_bound_step1}, we have 
    \begin{align}
        \alpha_j^* \le \frac{\wh q_t(\wh c)}{\wh q_{t,j}(\wh c)} + \frac{1}{\wh q_{t,j}(\wh c)}\left( \sqrt{\frac{\log(4k/\delta)}{2m}} + (5+2\gamma)\sqrt{\frac{\log(4k/\delta)}{2n \cdot p_s(y = j)}}\right) \,. 
    \end{align}
    Moreover, as $\wh c \ge c^\prime$,  we have $\wh q_{t,j}(\wh c) \ge \frac{\gamma}{2 + \gamma} \wh q_{t,j}(c^*_j)$ and hence, 
    \begin{align}
        \alpha_j^* - \frac{\gamma + 2}{\gamma \wh q_{t,j}(c^*_j)}\left( \sqrt{\frac{\log(4k/\delta)}{2m}} + (5+2\gamma)\sqrt{\frac{\log(4k/\delta)}{2n \cdot p_s(y = j)}}\right) \le  \frac{\wh q_t(\wh c)}{\wh q_{t,j}(\wh c)} = \wh \alpha_j \,. 
    \end{align}
    As we assume $n \cdot p_s(y = j) \ge \frac{2\log(4k/\delta)}{q_{t,j}^2(c^*_j)}$, we have $\wh q_{t,j}(c^*_j) \le q_{t,j}(c^*_j)/ 2$, which implies the following lower bound on $\alpha$: 
    \begin{align}
        \alpha_j^* - \frac{2\gamma + 4}{\gamma q_{t,j}(c^*_j)}\left( \sqrt{\frac{\log(4k/\delta)}{2m}} + (5+2\gamma)\sqrt{\frac{\log(4k/\delta)}{2n \cdot p_s(y = j)}}\right) \le \wh \alpha_j \,.  \label{eq:lower_bound}
    \end{align}
    
    Combining lower bound \eqref{eq:lower_bound} and upper bound \eqref{eq:upper_bound}, we get 
    \begin{align}
        \abs{\wh \alpha_j - \alpha_j^*} \le {l_j}\left( \sqrt{\frac{\log(4k/\delta)}{2m}} + \sqrt{\frac{\log(4k/\delta)}{2n \cdot p_s(y = j)}}\right) \,,
    \end{align}
    for some constant $l_j$. Additionally by our assumption of OSLS problem $p_s( y= j) > c/k$ for some constant $c > 0$, we have 
    \begin{align}
        \abs{\wh \alpha_j - \alpha_j^*} \le {l^\prime_j}\left( \sqrt{\frac{\log(4k/\delta)}{2m}} + \sqrt{\frac{k\log(4k/\delta)}{2n}}\right) \,,
    \end{align}
    for some constant $l_j^\prime$. 

    Combining the above obtained bound for all $j \in \out_s$ with union bound, we get with probability at least $1-\delta$,  
    \begin{align}
        \sum_{j \in \out_s} \abs{\wh \alpha_j - \alpha_j^*} \le {l^\prime_\text{max}} \left( \sqrt{\frac{k^2\log(4k/\delta)}{2m}} + \sqrt{\frac{k^3\log(4k/\delta)}{2n}}\right) \,,
    \end{align}
    where $l_\text{max}^\prime = \max l_j^\prime$. Now, note that for each $j \in \out_s$, we have $q_{t}(c) = p_t(y = j) \cdot q_{t, j}(c) + (1 - p_t(y = j)) \cdot q_{t, -j}(c)$. Hence $\alpha_j^* = p_t(y = j) + (1 - p_t(y = j)) \cdot q_{t, -j}(c)/\cdot q_{t, j}(c)$. Plugging this in, we get the desired bound. 
\end{proof}

\begin{figure*}[t!]
  \centering 
  \subfigure{\includegraphics[width=0.5\linewidth]{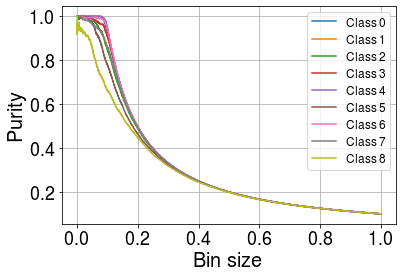}}
  \caption{Purity and size (in terms of fraction of unlabeled samples) in the top bin for all classes. 
  Bin size refers to the fraction of examples in the top bin. 
  With purity, we refer to the fraction of examples from a specific class $j$ in the top bin.   
  Results with ResNet-18 on CIFAR10 OSLS setup. Details of the setup in \appref{subsec:datasets_app}. As the bin size increases for all classes the purity decreases.   }
  \label{fig:loss_bin_property}
\end{figure*}

\update{Intuitively, the guarantees in the previous theorem capture the tradeoff due to the proportion of negative examples in the top bin (bias) versus the proportion of positives in the top bin (variance). As a corollary, we can show convergence to true mixture if there exits $c^*_j$ for all $j \in \out_s$ such that $q_{t, -j}(c^*_j, j) = 0$ and $q_{t, j}(c^*_j, j) \ge \epsilon$ for some $\epsilon > 0$. Put simply, efficacy of BBE relies on existence of a threshold on probability scores assigned by the classifier such that the examples mapped to a score greater than the threshold are *mostly* positive. Using the terminology from \citet{garg2021PUlearning}, we refer to this as the top bin property. Next, we provide empirical evidence of this property while using the source classifier to estimate the relative proportion of target label marginal among source classes.}

\textbf{Empirical evidence of the top bin property {} {}} 
We now empirically validate 
the positive pure top bin property
(\figref{fig:loss_bin_property}). We include results with Resnet-18 trained on the CIFAR10 OSLS setup same as our main experiments. 
We observe that source classifier approximately satisfies 
the positive pure top bin property for small enough top bin sizes. 


\subsection{Formal statement and proof of Theorem 2} \label{sec:proof_thm2}

In this section, we show that in population 
on a separable Gaussian dataset,
CVIR will recover the optimal classifier. 
Note that here we consider a binary classification 
problem similar to the one in Step 5 in \algoref{alg:PULSE}.
Since we are primarily interested in analysing 
the iterative procedure for obtaining domain discriminator classifier,
we assume that $\alpha$ is known.

In population, we have access to 
positive distribution (i.e., $p_p$), 
unlabeled distribution 
(i.e., $p_u \defeq \alpha p_p + (1-\alpha) p_n$), 
and mixture coefficient $\alpha$. Our goal is to recover 
the classifier that discriminates $p_p$ versus $p_n$.

For ease, we re-introduce some notation. 
For a classifier $f$ 
and loss function $\ell$, define 
\begin{align}
\vr_\alpha (f) = \inf \{ \tau \in \Real: \Prob_{x\sim p_u}( \ell(x, -1; f) \le \tau) \ge 1 - \alpha \} \,.     \label{eq:VIR_def}
\end{align}
Intuitively, $\vr_\alpha(f)$ identifies a 
threshold $\tau$ to capture bottom $1-\alpha$ fraction of the loss $\ell(x, -1)$ 
for points $x$ sampled from $p_u$.  
Additionally, define CVIR loss as 
\begin{align}
\calL (f, w) = \alpha \Expt{p_p}{\ell(x,1; f)} + \Expt{p_u}{w(x)\ell(x,-1; f)}\,,     \label{eq:CVIR_loss}
\end{align}
for classifier $f$ and some weights $w(x) \in\{0,1\}$.
Recall that given a classifier $f_t$ at an iterate $t$, CVIR procedure proceeds as follows: 
\begin{align}
    w_t(x) &= \indict{\ell(x, -1; f_t) \le \vr_\alpha(f_t)} \,,\label{eq:step1_CVIR_app}\\ 
    f_{t+1} &= f_t - \eta \grad \calL_f (f_t, w_t) \,. \label{eq:step2_CVIR_app}
\end{align}

We assume a data generating setup with where the 
support of positive and negative data is completely disjoint.
We assume that $x$ are drawn from two half multivariate 
Gaussian with mean zero and identity covariance, i.e., 
\begin{align*}
    x \sim p_p \Leftrightarrow x = \gamma_0 \theta_{\opt} + z | \, \theta_\opt^T z \ge 0,  \text{ where } z \sim \calN(0, I_d) \\ 
    x \sim p_n \Leftrightarrow x = -\gamma_0 \theta_{\opt} + z | \, \theta_\opt^T z < 0,  \text{ where } z \sim \calN(0, I_d) 
\end{align*}
Here $\gamma_0$ is the margin and $\theta_\opt \in \Real^d$ is the true separator. 
Here, we have access to distribution $p_p$ and $p_u = \alpha p_p + (1 - \alpha) p_n$. 
Assume $\ell$ as the logistic loss. 
For simplicity, we will denote $\calL(f_{\theta_t}, w_t)$ with $\calL(\theta_t, w_t)$. 

\begin{theorem}[Formal statement of \thmref{thm:informal_CVIR}] \label{thm:formal_CVIR}
In the data setup described above, a linear classifier 
$f(x; \theta) = \sigma\left(\theta^Tx \right)$ initialized at some $\theta_0$ such that $\calL(\theta_0, w_0) < \log(2)$,
trained with CVIR procedure as in equations 
\eqref{eq:step1_CVIR_app}-\eqref{eq:step2_CVIR_app} will converge to an optimal positive versus negative classifier. 
\end{theorem}

\begin{proof}[Proof of \thmref{thm:formal_CVIR}]
The proof uses two key ideas. One, at convergence of the CVIR procedure, 
the gradient of CVIR loss in \eqref{eq:CVIR_loss} converges to zero. 
\update{Second, for any classifier $\theta$ that is not optimal for positive versus negative classification, we show that
the CVIR gradient in \eqref{eq:CVIR_loss} is non-zero.} 

\textbf{Part 1 {} {}} We first show that the loss function $\calL( \theta, w)$ in \eqref{eq:CVIR_loss} 
is $2$-smooth with respect to $\theta$ for fixed $w$. Using gradient descent lemma with the decreasing property
of loss in \eqref{eq:step1_CVIR_app}-\eqref{eq:step2_CVIR_app}, we show that gradient converges to zero eventually. 
Considering gradient of $\calL$, we have
\begin{align}
\grad_\theta \calL( \theta, w )  = \alpha \Expt{p_p}{(f(x; \theta) - 1)x}  +\Expt{p_u}{ w(x)(f(x; \theta) - 0)x} \,.     
\end{align}
Moreover, $\grad^2 \calL$ is given by
\begin{align}
\grad_\theta^2 \calL( \theta, w )  =  \alpha \Expt{p_p}{\grad f(x; \theta) x x^T}  + \Expt{p_u}{ w(x) \grad f(x; \theta) x x^T} \,.     
\end{align}
Since $\grad f(x; \theta) \le 1$, we have $v^T \grad^2 \calL v \le 2$ for all unit vector $v \in R^d$. 
Now, by gradient descent lemma if $\eta \le \nicefrac{1}{2}$, at any step $t$ we have, 
$\calL(\theta_{t+1}, w_t) \le \calL(\theta_{t}, w_{t})$. Moreover, by definition of 
$\vr_\alpha(\theta)$ in \eqref{eq:VIR_def} and update \eqref{eq:step1_CVIR_app}, we have 
$\calL(\theta_{t+1}, w_{t+1}) \le \calL(\theta_{t+1}, w_{t})$. Hence, we have 
$\calL(\theta_{t+1}, w_{t+1}) \le \calL(\theta_{t}, w_{t})$. Since, the loss is lower 
bounded from below at $0$, for every $\epsilon > 0$, 
we have for large enough $t$ (depending on $\epsilon$),  
$\norm{\grad_\theta \calL(\theta_t, w_t)}{2} \le \epsilon$, i.e., 
$\norm{\grad_\theta \calL(\theta_t, w_t)}{2} \to 0$ as $t\to \infty$.

\textbf{Part 2 {} {}} Consider a general scenario when 
$\gamma > 0$. 
Denote the input domain of $p_p$ and $p_n$ as 
$P$ and $N$ respectively. 
At any step $t$, for all points $x \in \calX$ such that $p_u(x) > 0$ 
and $w_t(x) = 0$, we say that $x$ is rejected from $p_u$.  
We denote 
the incorrectly rejected subdomain of $p_n$ 
from $p_u$
as $N_r$ and the incorrectly accepted 
subdomain of $p_p$ from $p_u$ as $P_a$. Formally, 
$N_r = \{ x: p_n(x) > 0\text{ and }w_t(x) = 0 \}$
and 
$P_a = \{ x: p_p(x) > 0\text{ and }w_t(x) = 1 \}$. 
We will show that $p_p(P_a) \to 0$ as $t \to \infty$, 
%
and hence, we will recover the 
optimal classifier where we reject none of 
$p_u$ incorrectly.

Observe that at any time $t$, for fixed $w_t$ and $\theta = \theta_t$, the gradient of CVIR loss in 
\eqref{eq:CVIR_loss}, can be expressed as: 
\begin{align*}
\grad_\theta \calL( \theta, w_t )  = & \alpha \underbrace{ \int_{x\in P\setminus P_a } (f(x; \theta) - 1)x \cdot p_p(x) dx}_{\RN{1}}  +  (1- \alpha)\underbrace{ \int_{x\in N \setminus N_r } (f(x; \theta) - 0)x \cdot p_n(x) dx}_{\RN{2}} \\ & + \alpha  \underbrace{\int_{x\in P_a } (2f(x; \theta) - 1)x \cdot p_p(x) dx}_{\RN{3}}  \,. \label{eq:grad_CVIR} \numberthis
\end{align*}
Note that for any $x, \theta$,  $0 \le f(x; \theta) \le 1$.  Now consider inner product of individual terms above with $\theta_\opt$, we get 
\begin{align*}
    \inner{\RN{1}}{\theta_\opt} &=  \int_{x\in P\setminus P_a } (f(x; \theta) - 1)x^T \theta_\opt \cdot p_p(x) dx \le - \gamma_0 \int_{x\in P\setminus P_a } (1 - f(x; \theta)) \cdot p_p(x) dx\,, \label{eq:term1_ineq} \numberthis \\
    \inner{\RN{2}}{\theta_\opt} &= \int_{x\in N \setminus N_r } (f(x; \theta) - 0)x^T \theta_\opt \cdot p_n(x) dx \le -\gamma_0 \int_{x\in N \setminus N_r } (f(x; \theta) - 0) \cdot p_n(x) dx\,, \label{eq:term2_ineq} \numberthis \\
    \inner{\RN{3}}{\theta_\opt} &= \int_{x\in P_a } (2f(x; \theta) - 1)x^T \theta_\opt \cdot p_p(x) dx \le -\gamma_0 \int_{x\in P_a } ( 1- 2f(x; \theta)) \cdot p_p(x) dx \,. \label{eq:term3_ineq} \numberthis \\
\end{align*}

\update{Now, we will argue that individually all the three LHS terms in \eqref{eq:term1_ineq}, \eqref{eq:term2_ineq}, \eqref{eq:term3_ineq} are negative for all classifiers that do not separate positive versus negative data begining from $\calL(\theta_0, w_0) < \log(2)$. And hence, we show that these terms approach zero individually only when the linear classifier approaches an optimal positive versus negative classifier.}

\update{First, we consider the term in the LHS of equation \eqref{eq:term3_ineq}. When $\alpha = 0.5$, we have $\vr_\alpha(\theta) =  0.5$ and hence, $( 1- 2f(x; \theta)) \le 0$ for $x \in P_a$. When $\alpha > 0.5$, $\vr_\alpha(\theta) <  0.5$ because, the proportion $\alpha \cdot p_p(P_a)$ matches with proportion $(1-\alpha) \cdot p_n(N_r)$. Hence, we again have $( 1- 2f(x; \theta)) \le 0$ for $x \in P_a$.}

\update{To handle the case with $\alpha < 0.5$, we use a symmetry of he distribution to  because $\vr_\alpha(\theta) > 0.5$ and $(1 - 2f(x; \theta))$ can take positive and negative values. However, note that $\vr_\alpha(\theta)$ will be selected such that the proportion $\alpha \cdot p_p(P_a)$ matches with proportion $(1-\alpha) \cdot P_n(N_r)$. In particular, we can split $P_a$ into three disjoint sets $P_a^{(1)}$, $P_a^{(2)}$, and $P_a^{(3)}$ such that for all $x \in P_a^{(1)}$ we have $f(x; \theta) >= 0.5$, for all $x \in P_a^{(2)} \cup P_a^{(3)}$ we have $f(x; \theta) < 0.5$ and $p_p(P_a^{(3)}) = \frac{\alpha}{1-\alpha} p_p(N_r)$. Additionally, by symmetry of distribution around $\theta$, we have $\int_{x\in P_a^{(1)}} ( 1- 2f(x; \theta)) \cdot p_p(x) dx + \int_{x\in P_a^{(2)}} ( 1- 2f(x; \theta)) \cdot p_p(x) dx = 0$. Hence, we get
\begin{align*}
    \inner{\RN{3}}{\theta_\opt} &\le -\gamma_0 \int_{x\in P_a } ( 1- 2f(x; \theta)) \cdot p_p(x) dx = -\gamma_0 \int_{x\in P_a^{(3)}} ( 1- 2f(x; \theta)) \cdot p_p(x) dx \,. \label{eq:term4_ineq} \numberthis
\end{align*}
Combining all three cases, we get $\inner{\RN{3}}{\theta_\opt} < 0$ when $p_p(P_a) > 0$. 
}

\update{Now we consider LHS terms in \eqref{eq:term1_ineq} and \eqref{eq:term2_ineq}. Note that for all $x \in P \cup N$,  we have $0 \le f(x) \le 1$. Thus with $p_p( P \setminus P_a) > 0$, $\inner{\RN{1}}{\theta_\opt} \to 0$ when $f(x, \theta) \to 1$ for all $x \in P \setminus P_a$. Similarly with $p_n( N \setminus N_r) > 0$, $\inner{\RN{2}}{\theta_\opt} \to 0$ when $f(x, \theta) \to 0$ for all $x \in N \setminus N_r$.}


From part 1, for gradient $\norm{\grad_\theta \calL( \theta_t, w_t)}{2}$ to converge to zero as $t \to \infty$, we must have that LHS in equations \eqref{eq:term1_ineq}, \eqref{eq:term2_ineq}, 
and \eqref{eq:term3_ineq} converges to zero individually. 
Since CVIR loss decreases continuously and
$\calL(\theta_0, w_0) < \log(2)$, we have that $p_p(P_a) \to 0$ and hence,  $f(x, \theta) \to 1$ for all $x \in P$ and $f(x, \theta) \to 0$ for all $x \in N$.

%
%
\end{proof}

The above analysis can be extended to show convergence to max-margin classifier by using arguments from 
\citet{soudry2018implicit}. In particular, as $p_p(P_a) \to 0$, we can show that 
$\theta_t/\norm{\theta_t}{2}$ will converge to the max-margin classifier 
for $p_p$ versus $p_n$, i.e.,  $\theta_\opt$ if $p_p(P_a) \to 0$ in finite number of steps. 
Note that we need an assumption that the initialized model 
$\theta_0$ is strictly better than a model that randomly guesses or initialized at all zeros. 
This is to avoid convergence to the local minima of $\theta = \bf{0}$ with CVIR training. 
This assumption 
is satisfied when the classifier is initialized in a way such that 
$\inner{\theta_0}{\theta_\opt} > 0$. In general, we need 
a weaker assumption that during training with any randomly 
initialized classifier, there exists an iterate $t$ during
CVIR training such that $\inner{\theta_t}{\theta_\opt}  > 0$. 

\subsection{Extension of Theorem 1} \label{sec:BBE_ext_app}

We  also extend the analysis in the proof of \thmref{thm:formal_BBE} to Step 5 of \algoref{alg:PULSE}
to show convergence of estimate $\wh p_t(y=k+1)$ to true prevalence
$p_t(y = k+1)$. In particular, we show that the estimation error for prevalence of the novel class will primarily depend on sum of two terms: (i) error in approximating the label shift corrected source distribution, i.e., $p_s^\prime(x)$; and (ii) purity of the top bin of the domain discriminator classifier.

Before formally introducing the result, we introduce some notation. Similar to before, 
given probability density function $p$ 
and a domain discriminator classifier $f : \calX \to \Delta$,
define a function $q = \int_{ A(z)} p(x) dx$, 
where $A(z) = \{x \in \inpt: f(x) \ge z\}$ for all $z\in [0,1]$. 
Intuitively, $q(z)$ captures the 
cumulative density of points in a top bin, i.e., 
the proportion of input domain 
that is assigned a value larger than $z$ 
by the function $f$
in the transformed space. 
We denote  $p_t(x| y=k+1)$ with $p_{t, k+1}$.
For each pdf $p_t$, $p_{t, k+1}$, and $p_s^\prime$, we define 
$q_t$, $q_{t,k+1}$, and $q_s^\prime$ respectively. 
Note that since
We define an empirical estimator $\wh q(z)$
given a set $X = \{x_1, x_2, \ldots, x_n\}$
sampled iid from $p(x)$. Let $Z = f(X)$. 
Define $ \wh q(z) = \sum_{i=1}^n \indict{z_i \ge z}/{n}$.

Recall that in Step 5 of \algoref{alg:PULSE}, 
to estimate the proportion of novel class, 
we have access to re-sampled data from approximate label shift
corrected source distribution $\wh q_s^\prime(x)$. 
Assume that we the size of re-sampled dataset 
is $n$. 


\begin{theorem} \label{thm:BBE_step5}
Define $c^* = \argmin_{c \in [0,1]} \left( {q_{t, k+1}(c)}/{\wh q_{s}^\prime(c)} \right)$. 
Assume $\min(n, m) \ge \left( \frac{2\log(4/\delta)}{({\wh q_{s}^\prime}(c^*))^2} \right)$.
Then, for every $\delta > 0$, $[\wh p_t]_{k+1} \defeq \wh p_t (y = k+1)$ in Step 5 of \algoref{alg:PULSE} satisfies with probability 
at least $1-\delta$, we have: 
\begin{align*}
    \abs{[\wh p_t]_{k+1} - [p_t]_{k+1}} \le \left(1 - [p_t]_{k+1}\right) & \underbrace{\frac{\abs{q_s^\prime(c^*) -  \wh q_s^\prime(c^*)}}{\wh q_{s}^\prime(c^*)}}_{\substack{\text{Error in estimating} \\ \text{label shift corrected source}}} +  [p_t]_{k+1} \underbrace{\left( \frac{q_{t, k+1}(c^*)}{\wh q_{s}^\prime(c^*)} \right)}_{\substack{\text{Impurity in} \\ \text{top bin}}} \\ & +  \calO\left(\sqrt{\frac{\log(4/\delta)}{n}} + \sqrt{\frac{\log(4/\delta)}{m}}\right)  \,.
\end{align*}
%
\end{theorem}
\begin{proof}
We can simply prove this theorem as Corollary of Theorem 1 from \citet{garg2021PUlearning}. Note that $q_t(c^*) = ( 1 - p_t(y = k+1 )) \cdot  q^\prime_{s}(c^*) + p_t(y = k+1) \cdot q_{t, k+1} (c^*)$. Adding and subtracting $( 1 - p_t(y = k+1 )) \cdot  \wh q^\prime_{s}(c^*)$ and dividing by $\wh q^\prime_s$, we get $\frac{q_t(c^*)}{\wh q^\prime_s(c^*)} = ( 1 - p_t(y = k+1 )) \cdot \frac{\abs{q_s^\prime(c^*) -  \wh q_s^\prime(c^*)}}{\wh q_{s}^\prime(c^*)} + ( 1 - p_t (y = k+1)) +   p_t(y = k+1) \cdot \frac{q_{t, k+1}(c^*)}{\wh q_s^\prime(c^*)}$. Plugging in bound for LHS from Theorem 1 in \citet{garg2021PUlearning}, we get the desired result. 
\end{proof}

\subsection{Extensions of Theorem 2 to general separable datasets} \label{sec:CVIR_ext_app}

For general separable datasets, CVIR has undesirable property of getting stuck at local optima where gradient in  \eqref{eq:term3_ineq} can be zero by maximizing entropy on the subset $P_a$ which is (incorrectly) not-rejected from $p_u$ in CVIR iterations. Intuitively, if the classifier can perfectly separate $P\setminus P_a$ and $N \setminus N_r$ and at the same time maximize the entropy of the region $P_a$, then the classifier trained with CVIR can get stuck in this local minima. 

However, we can extend the above analysis with some modifications to the CVIR procedure. Note that when the CVIR classifier maximizes the entropy on $P_a$. it makes an error on points in $P_a$. Since, we have access to the distribution $p_p$, we can add an additional regularization penalty to the CVIR loss that ensures that the converged classifier with CVIR correctly classifies all the points in $p_p$. With a large enough regularization constant for the supervised loss on $p_p$, we can dominate the gradient term in \eqref{eq:term3_ineq} which pushes CVIR classifier to correct decision boundary even on $P_a$ (instead of maximizing entropy). We leave formal analysis of this conjecture for future work. 
Since we warm start CVIR training with a positive versus unlabeled classifier, if we obtain an initialization close enough to the true positive versus negative decision boundary, by monotonicity property of CVIR iterations, we may never get stuck in such a local minima even without modifications to loss. 


\begin{figure*}[t!]
  \centering 
  \subfigure[Accuracy on validation positive versus negative data]{\includegraphics[width=0.3\linewidth]{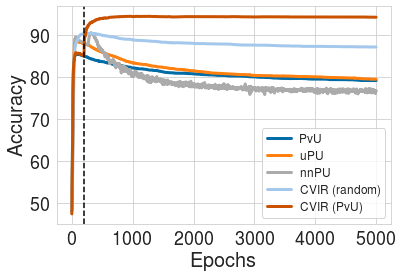}} \hfil
  \subfigure[Fraction of correctly rejected examples with CVIR]{\includegraphics[width=0.3\linewidth]{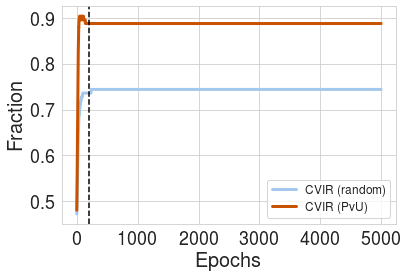}}
  \caption{\textbf{Comparison of different methods in overparameterized toy setup.} CVIR (random) denotes CVIR with random initialization and CVIR (PvU) denotes warm start with a positive versus negative classifier. Vertical line denotes the epoch at which we switch from PvU to CVIR in CVIR (PvU) training. (a) We observe that CVIR (PvU) improves significantly even over the best early stopped PvU model. As training proceeds, we observe that accuracy of nnPU, uPU and PvU training drops whereas CVIR (random) and CVIR (PvU) maintains superior and stable performance. (b) We observe that warm start training helps CVIR over randomly initialized model to correctly identity positives among unlabeled for rejection. }
  \label{fig:empirical_val}
\end{figure*}
\section{Empirical investigation of CVIR in toy setup}
\label{sec:empirical_CVIR_app}

As noted in our ablation experiments and in \citet{garg2021PUlearning}, 
domain discriminator trained with CVIR outperforms classifiers 
trained with other consistent objectives (nnPU~\citep{kiryo2017positive} and uPU~\citep{du2015convex}). 
While the analysis in \secref{sec:analysis} highlights consistency 
of CVIR procedure
in population, 
it doesn't capture the observed 
empirical efficacy of CVIR over 
alternative methods 
in overparameterized models. 
In the Gaussian setup described 
in \secref{sec:proof_thm2}, 
we train overparameterized 
linear models to compare CVIR 
with other methods (\figref{fig:empirical_val}). 
We fix $d=1000$
and use $n=250$ positive and $m= 250$ unlabeled points 
for training with $\alpha = 0.5$. We set the margin 
$\gamma$ at $0.05$. 
We compare CVIR with unbiased losses uPU and nnPU. 
We also make comparison 
with a naive positive versus unlabeled classifier 
(referred to as PvU). 
For CVIR, we experiment with a randomly 
initialized classifier and initialized with a
PvU classifier trained for $200$ epochs.

First, we observe that when a classifier 
is trained to distinguish positive and unlabeled
data, \emph{early learning}
happens~\citep{liu2020early,arora2019fine,garg2021RATT}, 
i.e., during the initial phase of learning 
classifier learns to classify 
positives in unlabeled correctly as positives
achieving high accuracy on validation 
positive versus negative data. 
While the early learning happens with all methods, 
soon in the later phases of training PvU starts overfitting 
to the unlabeled data as negative hurting its validation 
performance. 
For uPU and nnPU, while they improve over 
PvU training during the initial epochs, 
the loss soon becomes biased 
hurting the performance of classifiers trained with 
uPU and nnPU
on validation data.

For CVIR trained from a randomly initialized classifier, 
we observe that it improves slightly over the best PvU or 
the best nnPU model. Moreover, it maintains a relatively stable 
performance throughout the training. 
CVIR initialized with a PvU classifier 
significantly improves the performance. 
In \figref{fig:empirical_val} (b), 
we show that CVIR initialized with a PvU
correctly rejects 
significantly more fraction of positives 
from unlabeled than CVIR trained from scratch. 
Thus, post early learning  
rejection of large fraction of positives
from unlabeled training  
in equation \eqref{eq:step1_CVIR}
crucially helps CVIR. 

\section{Experimental Details} \label{sec:exp_app}

\subsection{Baselines} \label{subsec:baselines_app}

We compare PULSE with several popular methods from OSDA literature. While these methods are not specifically 
proposed for OSLS, they are introduced for the more general OSDA problem. In particular, we make comparions with DANCE~\citep{saito2020universal}, UAN~\citep{you2019universal}, CMU~\citep{fu2020learning}, STA~\citep{liu2019separate}, Backprop-ODA (or BODA)~\citep{saito2018open}. We use the open source implementation available at \url{https://github.com/thuml} and \url{https://github.com/VisionLearningGroup/DANCE/}.
Since OSDA methods do not estimate the prevalence of novel class explicitly,
we use the fraction of examples predicted in class $k+1$ as a surrogate. We next briefly describe the main idea for each method:

\emph{Backprob-ODA {} {}} \citet{saito2018open} proposed backprob ODA to train a $(k+1)$-way classifier. In particular, the network is trained to correctly classify source samples and for target samples, the classifier (specifically the last layer) is trained to output $0.5$ for the probability of the unknown class. The feature extractor is trained adversarially to move the probability of unknown class away from $0.5$ on target examples by utilizing the gradient reversal layer. 

\emph{Separate-To-Adapt (STA) {} {}} \citet{liu2019separate} trained a network that learns jointly from source and target by learning to separate negative (novel) examples from target. The training is divided into two parts. The first part consists of training a multi-binary $G_c|^{\abs{\out_s}}_{c = 1}$ classifier on labeled source data for each class and a binary classifier $G_b$ which generates the weights $w$ for rejecting target samples in the novel class. The second part consists of feature extractor $G_f$, a classifier $G_y$ and domain discriminator $G_d$ to perform adversarial domain adaptation between source and target data in the source label space. $G_y$ and $G_d$ are trained with incorporating weights $w$ predicted by $G_b$ in the first stage. 

\emph{Calibrated Multiple Uncertainties (CMU) {} {}} \citet{fu2020learning} trained a source classifier and a domain discriminator to discriminate the novel class from previously seen classes in target. To train the discriminator network, CMU uses a weighted binary cross entropy loss where $w(x)$ for each example $x$ in target which is the average of uncertainty estimates, e.g. prediction confidence of source classifier. During test time, target data $x$ with $w(x) \ge w_0$ (for some pre-defined threshold $w_0$) is classified as an example from previously seen classes and is given a class prediction with source classifier. Otherwise, the target example is classified as belonging to the novel class. 

\emph{DANCE {} {}} \citet{saito2020universal} proposed DANCE which combines a self-supervised clustering loss to cluster neighboring target examples and an entropy separation loss to consider alignment with source. Similar to CMU, during test time, DANCE uses thresholded prediction entropy of the source classifier to classifier a target example as belonging to the novel class. 

\emph{Universal Adaptation Networks (UAN) {} {}}  \citet{you2019universal} proposed UAN which also trains a source classifier and a domain discriminator to discriminate the novel class from previously seen classes in target. The objective is similar to CMU where instead of using uncertainty estimates from multiple classifiers, UAN uses prediction confidence of domain discriminator classifier. Similar to CMU, at test time, target data $x$ with $w(x) \le w_0$ (for some pre-defined threshold $w_0$) is classified as an example from previously seen classes and is given a class prediction with source classifier. Otherwise, the target example is classified as belonging to the novel class.

For alternative baselines, we experiment with source classifier directly deployed on the target data which may contain novel class and label shift among source classes (referred to as \emph{source-only}). This naive comparison is included to quantify benefits of label shift correction and identifying novel class over a typical $k$-way classifiers.

We also train a domain discriminator classifier for source versus target (referred to as \emph{domain disc.}). This is an adaptation of PU learning baseline\citep{elkan2008learning} which assumes no label shift among source classes. We use simple
domain discriminator training to distinguish source versus target. To estimate the fraction of novel examples, 
we use the EN estimator proposed in \citet{elkan2008learning}. 
For any target input, we make a prediction with the 
domain discriminator classifier (after re-scaling the sigmoid output with
the estimate proportion of novel examples). 
Any example that is classified as target, we assign it the class $k+1$. 
For examples classified as source, we make a prediction for them using 
the $k$-way source classifier. 

Finally, per the reduction presented in \secref{sec:reduction}, we train $k$ PU classifiers (referred to as \emph{k-PU}).
To train each PU learning classifier, we can plugin any method discussed in \secref{sec:prelims}. 
In the main paper, we included results obtained with plugin state-of-the-art PU learning algorithms. 
In \appref{subsec:different_PU_app}, we present ablations with other PU learning methods. 


\subsection{Dataset and OSLS Setup Details} \label{subsec:datasets_app}

We conduct experiments with seven benchmark classification datasets across 
vision, natural language, biology and medicine. 
Our datasets span language, image and table modalities. 
For each dataset, we simulate an OSLS problem.
We experiment with different fraction of novel class prevalence, 
source label distribution, and target label distribution. 
We randomly choose classes that constitute the novel target class. 
After randomly choosing source and novel classes, we first split the training data from each source class 
randomly into two partitions. This creates a random label distribution for shared classes among source and target. We then club novel classes to assign them a new class (i.e. $k+1$). 
Finally, we throw away labels for the target data to obtain an unsupervised DA problem. We repeat the same process on iid hold out data to obtain validation data with no target labels. 
For main experiments in the paper, we next describe important details for the OSLS setup simulated. 
All the other details can be found in the code repository.

For vision, we use CIFAR10, CIFAR100 ~\citep{krizhevsky2009learning} and Entity30~\citep{santurkar2020breeds}.  For language, we experiment with Newsgroups-20
dataset. 
Additionally, inspired by applications of OSLS in biology and medicine, we experiment with Tabula Muris~\citep{tabula2020single} (Gene Ontology prediction), Dermnet (skin disease prediction), and BreakHis~\citep{spanhol2015dataset} (tumor cell classification). 

\textbf{CIFAR10 {} {}} For CIFAR10, we randomly select $9$ classes as the source classes and a novel class formed by the remaining class. After randomly sampling the label marginal for source and target randomly, we get the prevalence for novel class as $0.2152$. 

\textbf{CIFAR100 {} {}} For CIFAR100, we randomly select $85$ classes as the source classes and a novel class formed by aggregating the data from $15$ remaining classes. After randomly sampling the label marginal for source and target randomly, we get the prevalence for novel class as $0.2976$. 

\textbf{Entity30 {} {}} Entity30 is a subset of ImageNet~\citep{russakovsky2015imagenet} with 30 super classes. For Entity30, we randomly select $24$ classes as the source classes and a novel class formed by aggregating the data from $6$ remaining classes. After randomly sampling the label marginal for source and target randomly, we get the prevalence for novel class as $0.3942$.

\textbf{Newgroups-20}  For Newsgroups20\footnote{\url{http://qwone.com/~jason/20Newsgroups/}}, we randomly select $16$ classes as the source classes and a novel class formed by aggregating the data from $4$ remaining classes. After randomly sampling the label marginal for source and target randomly, we get the prevalence for novel class as $0.3733$. 
%
This dataset is motivated by scenarios where novel news categories can appear over time but the distribution of articles given a news category might stay relatively unchanged. 

\textbf{BreakHis {} {}} BreakHis\footnote{\url{https://web.inf.ufpr.br/vri/databases/breast-cancer-histopathological-database-breakhis/}} contains $8$ categories of cell types, $4$ types of benign breast tumor and $4$ types malignant tumors (breast cancer). Here, we simulate OSLS problem specifically where $6$ cell types are observed in the source ($3$ from each) and a novel class appears in the target with $1$ cell type from each category.  After randomly sampling the label marginal for source and target randomly, we get the prevalence for novel class as $0.2708$.

\textbf{Dermnet {} {}} Dermnet data contains images of 23 types of skin diseases taken from Dermnet NZ\footnote{\url{http://www.dermnet.com/dermatology-pictures-skin-disease-pictures}}. We simulate OSLS problem specifically where $18$ diseases are observed in the source and a novel class appears in the target with the rest of the $5$ diseases. After randomly sampling the label marginal for source and target randomly, we get the prevalence for novel class as $0.3133$.

\textbf{Tabula Muris {} {}}
Tabula Muris dataset~\citep{tabula2020single} comprises of different cell types collected across $23$ organs of the mouse model organism. We use the data pre-processing scripts provided in \citep{cao2020concept}\footnote{\url{https://github.com/snap-stanford/comet}}. We just use the training set comprising of $57$ classes for our experiments. 
We simulate OSLS problem specifically where $28$ cell types are observed in the source and a novel class appears in the target with the rest of the $29$ cell types. After randomly sampling the label marginal for source and target randomly, we get the prevalence for novel class as $0.6366$.


\subsection{Details on the Experimental Setup} \label{subsec:setup_app}

We use Resnet18~\citep{he2016deep} for CIFAR10, CIFAR100, and Entity30. 
For all three datasets, in our main experiments, we train Resnet-18 from scratch.
We use SGD training with momentum of $0.9$ for $200$ epochs. 
We start with learning rate $0.1$ and 
decay it by multiplying it with $0.1$ every $70$ epochs. 
We use a weight decay of $5\times 10^{-4}$. 
For CIFAR100 and CIFAR10, we use batch size of $200$. 
For Entity30, we use a batch size of $32$.  
In \appref{subsec:contrastive_app}, we experiment with contrastive pre-training instead of random initialization.

For newsgroups, we use a convolutional architecture\footnote{\url{https://github.com/mireshghallah/20Newsgroups-Pytorch}}. 
We use glove embeddings to initialize the embedding layer. 
We use Adam optimizer with a learning rate of $0.0001$ and no weight decay. 
We use a batch size of $200$. We train with constant learning rate for $120$ epochs.

For Tabular Muris, we use the fully connected MLP used in \citet{cao2020concept}. 
We use the hyperparameters used in \citet{cao2020concept}. 
We use Adam optimizer with a learning rate of $0.0001$ and no weight decay.
We train with constant learning rate for $40$ epochs. We use a batch size of $200$. 

For Dermnet and BreakHis, we use Resnet-50 pre-trained on Imagenet. 
We use an initial learning rate of $0.0001$ and decay 
it by $0.96$ every epoch. 
We use SGD training with momentum of $0.9$ 
and weight decay of $5\times 10^{-4}$. 
We use a batch size of $32$. 
These are the default hyperparameters used in 
\citet{alom2019breast} and \citet{liao2016deep}.

For all methods, we use the same backbone for discriminator and source classifier. 
Additionally, for PULSE and domain disc.,
we use the exact same set of hyperparameters to train the domain discriminator and source classifier. 
For kPU, we use a separate final layer for each class with the same backbone. 
We use the same hyperparameters described above for all three methods. 
For OSDA methods, we use default method specific 
hyperparameters introduced in their works.
Since we do not have access to labels from the target data, 
we do not perform hyperparameter tuning but instead use the standard hyperparameters 
used for training on labeled source data.  
In future, we may hope to leverage  heuristics proposed 
for accuracy estimation without access to labeled target data~\citep{garg2022ATC}. 

We train models till the performance on validation source data (labeled) ceases to increase. 
Unlike OSDA methods, note that we do not use early stopping based on performance on held-out labeled target data.  
To evaluate classification performance, we report target accuracy on all classes, seen classes and the novel class. For target marginal, we separately report estimation error for previously seen classes and for the novel class. 
For the novel class, we report absolute difference between true and estimated marginal. For seen classes, we report average absolute estimation error.   
We open-source our code 
at \url{https://github.com/Neurips2022Anon}. 
By simply changing a single config file, new OSLS setups can be generated and experimented with. 

Note that for our main experiments, for vision datasets (i.e., CIFAR10, CIFAR100, and Entity30) and for language dataset, we do not 
initialize with a (supervised) pre-trained model to avoid overlap of novel classes with the classes in the dataset used for pre-training. 
For example, labeled Imagenet-1k is typically used for pre-training. However, Imagenet classes overlaps with all three vision datasets employed 
and hence, we avoid pre-trained initialization. 
In \appref{subsec:contrastive_app}, we experiment with contrastive pre-training on Entity30 and CIFAR100. 
In contrast, for medical datasets, we leverage Imagenet pre-trained models as there is 
no overlap between classes in BreakHis and Dermnet with Imagenet. 

\subsection{Detailed results from main paper} \label{subsec:results_paper_app}

For completeness, we next include results for all datasets. In particular, for each dataset we tabulate (i) overall accuracy on 
target; (ii) accuracy on seen classes in target; (iii) accuracy on the novel class; (iv) sum of absolute error in estimating target marginal among previously seen classes, i.e., $\sum_{y\in \out_s}\abs{\wh p_t(y) - p_t(y)}$;  and (v) absolute error for novel fraction estimation, i.e., $\abs{\wh p_t(y = k+1} - p_t(y = k+1)$.
\tabref{table:results_full} presents results on all the datasets. \figref{fig:accuracy_plot} and \figref{fig:mpe_plot} presents epoch-wise results. 

\subsection{Investigation into OSDA approaches} \label{subsec:OSDA_app}

We observe that with
default hyperparameters, 
popular OSDA methods significantly 
under perform as compared to PULSE. 
We hypothesize that the primary reasons 
underlying the poor performance 
of OSDA methods are (i) the heuristics 
employed to detect novel classes; and 
(ii) loss functions incorporated to 
improve alignment between examples from 
common classes in source and target.
To detect novel classes, 
a standard heuristic employed 
popular OSDA methods involves 
thresholding uncertainty 
estimates (e.g., prediction entropy, 
softmax confidence~\citep{you2019universal,fu2020learning,saito2020universal}) 
at a predefined 
threshold $\kappa$. However, 
a fixed $\kappa$, may not 
for different datasets and 
different fractions of 
the novel class. 
Here, 
we ablate by (i) removing loss function terms 
incorporated with an aim to improve source 
target alignment; and (ii) vary threshold
$\kappa$ and show improvements in performance 
of these methods.

For our investigations, we experiment with CIFAR10, with UAN and DANCE methods. 
For DANCE, we remove the entropy separation loss employed to encourage align target examples with source examples. 
For UAN, we remove the adversarial
domain discriminator training employed to align target examples with source examples. 
For both the methods, we observe that by removing the corresponding loss function terms we obtain a marginal improvement. For DANCE on CIFAR10, the performance goes up from $70.4$ to $72.5$ (with the same hyperparameters as the default run). FOR UAN, we observe similar minor improvements, where the performance goes up from $15.4$ to $19.6$. 

Next, we vary the threshold used for detecting the novel examples. By optimally tuning the threshold for CIFAR10 with UAN, we obtain a substantial increase. In particular, the overall target accuracy increases from $19.6$ to $33.1$. With DANCE on CIFAR10, optimal threshold achieves $75.6$ as compared to the default accuracy $70.4$.
In contrast, our two-stage method 
PULSE avoids the need to guess $\kappa$, by 
first estimating 
the fraction of novel class which 
then guides the classification 
of novel class versus previously seen classes.

\subsection{Ablation with novel class fraction} \label{subsec:novel_class_app}

In this section, we ablate on novel class proportion on CIFAR10, CIFAR100 and Newsgroups20. For each dataset we experiment with three settings, each obtained by varying the number of classes from the original data that constitutes the novel classes. We tabulate our results in \tabref{table:results_ablations}. 

\subsection{Contrastive pre-training on unlabeled data} \label{subsec:contrastive_app}

Here, we experiment with contrastive pre-training to pre-train the backbone networks used for feature extraction. In particular, 
we initialize the backbone architectures with SimCLR pre-trained weights. We experiment with CIFAR100 and Entity30 datasets. Instead of pre-training on mixture of source and target unlabeled data, we leverage the publicly available pre-trained weights\footnote{For CIFAR100: \url{https://drive.google.com/file/d/1huW-ChBVvKcx7t8HyDaWTQB5Li1Fht9x/view} and for Entity30, we use Imagenet pre-trained weights from here: \url{https://github.com/AndrewAtanov/simclr-pytorch}.}. \tabref{table:results_pretraining} summarizes our results. 
We observe that pre-training improves over random initialization for all the methods with PULSE continuing to outperform other approaches.

\begin{table}[h]
  \centering
  \small
  \setlength{\tabcolsep}{1.5pt}
  \renewcommand{\arraystretch}{1.2}
  \caption{{Comparison with different OSLS approaches with pre-trained feature extractor}. We use SimCLR pre-training to initialize the feature extractor for all the methods. All methods improve over random initialization (in \tabref{table:results}). Note that PULSE continues to outperform other approaches.}\label{table:results_pretraining}
  \begin{tabular}{@{}*{9}{c}@{}}
  \toprule
   &&    
   \multicolumn{2}{c}{\textbf{CIFAR100}}  & &  
   \multicolumn{2}{c}{\textbf{Entity30}} \\
    \multirow{2}{*}{Method} 
     &&    \multirow{2}{*}{  \parbox{1.0cm}{\centering  Acc (All)}}  &  \multirow{2}{*}{  \parbox{1.cm}{\centering MPE (Novel)}} 
      & &  \multirow{2}{*}{  \parbox{1.0cm}{\centering  Acc (All)}}  &  \multirow{2}{*}{  \parbox{1.cm}{\centering MPE (Novel)}}\\
      & & &  & &  \\
  \midrule
  BODA~\citep{saito2018open} && $37.1$ & $0.34$ && $52.1$ & $0.376$  \\ 
  Domain Disc. && $49.4$ & $0.041$ && $57.4$ & $0.024$    \\ 
  kPU && $37.5$ & ${0.297}$  && ${70.1}$ & ${0.32}$ \\ 
  PULSE (Ours) && $67.3$ & ${0.052}$  && $72.4$ & ${0.002}$   \\
  \bottomrule 
  \end{tabular}  
  \vspace{-5pt}
\end{table}

\subsection{Ablation with different PU learning methods} \label{subsec:different_PU_app}

In this section, we experiment with alternative PU learning approaches for PULSE and kPU. 
In particular, we experiment with the next best alternatives, i.e.,  
nnPU instead of CVIR for classification and DEDPUL instead of BBE for target marginal estimation.
We refer to these as kPU (alternative) and PULSE (alternative) in \tabref{table:results_ablation}. 
%
We present results on three datasets: CIFAR10, CIFAR100 and Newsgroups20 in the same setting as described in 
\secref{subsec:datasets_app}. 
We make two key observations: (i) PULSE continues to dominate kPU with alternative choices; (ii) CVIR and BBE significantly outperform alternative choices.

\begin{table}[h]
  \centering
  \small
  \setlength{\tabcolsep}{1.5pt}
  \renewcommand{\arraystretch}{1.2}
  \caption{{Comparison with different PU learning approaches}. `Alternative' denotes results with employing nnPU for classification and DEDPUL for target marginal estimation instead of `default' which uses CVIR and BBE. }\label{table:results_ablation}
  \begin{tabular}{@{}*{9}{c}@{}}
  \toprule
   &   
   \multicolumn{2}{c}{\textbf{CIFAR10}} & & 
   \multicolumn{2}{c}{\textbf{CIFAR100}}  & &  
   \multicolumn{2}{c}{\textbf{Newsgroups20}} \\
    \multirow{2}{*}{Method} 
     &    \multirow{2}{*}{  \parbox{1.0cm}{\centering  Acc (All)}}  &  \multirow{2}{*}{  \parbox{1.cm}{\centering MPE (Novel)}} 
      & &  \multirow{2}{*}{  \parbox{1.0cm}{\centering  Acc (All)}}  &  \multirow{2}{*}{  \parbox{1.cm}{\centering MPE (Novel)}} 
    & &   \multirow{2}{*}{  \parbox{1.0cm}{\centering  Acc (All)}}  &  \multirow{2}{*}{  \parbox{1.cm}{\centering MPE (Novel)}} \\
      & & &  & &  \\
  \midrule
  $k$-PU (alternative) & $53.4$ & $0.215$ && $12.1$ & $0.298$ && $14.1$ & $0.373$  \\ 
  $k$-PU (default) & $83.6$ & $0.036$ && $36.3$ & $0.298$ && $52.1$ & $0.307$   \\ 
  PULSE (alternative) & ${80.5}$ & ${0.05}$  && ${30.1}$ & ${0.231}$ && $39.8$ & $0.223$\\ 
  PULSE (default) & $86.1$ & ${0.008}$  && ${63.4}$ & ${0.078}$ && ${62.2}$ & ${0.061}$ \\ 
 
  \bottomrule 
  \end{tabular}  
  \vspace{-5pt}
\end{table}

\begin{table}[t]
  \centering
  \small
  \setlength{\tabcolsep}{1.5pt}
  \renewcommand{\arraystretch}{1.2}
  \caption{{Comparison with different OSLS approaches for different novel class prevalence}. We observe that for on CIFAR100 and Newsgroups20, PULSE maintains superior performance as compared to other approaches. On CIFAR10, as the proportion of novel class increases, the performance of of kPU improves slightly over PULSE for target accuracy. }\label{table:results_ablations}
  \begin{tabular}{@{}*{11}{c}@{}}
  \toprule
   & & 
    \multicolumn{2}{c}{ \multirow{2}{*}{ \parbox{3.cm}{\centering \textbf{CIFAR10} $(p_t(k+1) = 0.215)$} }}  && 
    \multicolumn{2}{c}{ \multirow{2}{*}{ \parbox{3.cm}{\centering \textbf{CIFAR10} $(p_t(k+1) = 0.406)$} }}  && 
    \multicolumn{2}{c}{ \multirow{2}{*}{ \parbox{3.cm}{\centering \textbf{CIFAR10} $(p_t(k+1) = 0.583)$} }}  \\
  & & & & &  \\
   {Method} 
     && Acc (All)  &    MPE (Novel)
      &&   Acc (All)  &    MPE (Novel) &&
       Acc (All)  &   MPE (Novel) \\
  \midrule
  BODA~\citep{saito2018open} && $63.1$ &  $0.162$ && $65.5$ & $0.166$ && $48.6$ & $0.265$  \\ 
  Domain Disc. && $47.4$ & $0.331$ && $57.5$ & $0.232$ &&  $68.7$ & $0.144$   \\ 
  kPU && $83.6$ &  $0.036$ && $87.8$ & $0.010$ && $89.9$ & $0.036$  \\ 
  PULSE (Ours) && ${86.1}$ & ${0.008}$  && $87.4$ & $0.009$ && $83.7$ & $0.006$  \\ 
 
  \bottomrule 
  \end{tabular}  
  \vspace{5pt}
  
  \begin{tabular}{@{}*{11}{c}@{}}
  \toprule
   & & 
    \multicolumn{2}{c}{ \multirow{2}{*}{ \parbox{3.cm}{\centering \textbf{CIFAR100} $(p_t(k+1) = 0.2976)$} }}  && 
    \multicolumn{2}{c}{ \multirow{2}{*}{ \parbox{3.cm}{\centering \textbf{CIFAR100} $(p_t(k+1) = 0.4477)$} }}  && 
    \multicolumn{2}{c}{ \multirow{2}{*}{ \parbox{3.cm}{\centering \textbf{CIFAR100} $(p_t(k+1) = 0.5676)$} }}  \\
  & & & & &  \\
   {Method} 
     && Acc (All)  &    MPE (Novel)
      &&   Acc (All)  &    MPE (Novel) &&
       Acc (All)  &   MPE (Novel) \\
  \midrule
  BODA~\citep{saito2018open} &&  $36.1$ & $0.41$ && $41.6$ & $0.075$ && $50.2$ & $0.03$  \\ 
  Domain Disc. && $45.8$ & ${0.046}$ &&  $52.3$ & $0.092$ && $58.7$ & $0.187$  \\ 
  kPU && $36.3$ & $0.298$ && $52.2$ & $0.448$ && $63.9$ & $0.568$ \\ 
  PULSE (Ours) &&  ${63.4}$  & $0.078$ && $66.6$ & $0.052$  && $68.2$ & $0.088$  \\ 
 
  \bottomrule 
  \end{tabular}
  \vspace{5pt}
  
  \begin{tabular}{@{}*{11}{c}@{}}
  \toprule
   & & 
    \multicolumn{2}{c}{ \multirow{2}{*}{ \parbox{3.cm}{\centering \textbf{Newsgroups20} $(p_t(k+1) = 0.3733)$} }}  && 
    \multicolumn{2}{c}{ \multirow{2}{*}{ \parbox{3.cm}{\centering \textbf{Newsgroups20} $(p_t(k+1) = 0.6452)$} }}  && 
    \multicolumn{2}{c}{ \multirow{2}{*}{ \parbox{3.cm}{\centering \textbf{Newsgroups20} $(p_t(k+1) = 0.7688)$} }}  \\
   & & & & &  \\
   {Method} 
     && Acc (All)  &    MPE (Novel)
      &&   Acc (All)  &    MPE (Novel) &&
       Acc (All)  &   MPE (Novel) \\
  \midrule
  BODA~\citep{saito2018open} && $43.4$ & $0.16$  && $25.5$ & $0.645$ && $17.7$ & $0.769$ \\ 
  Domain Disc. &&  $50.9$ & $0.176$  && $44.8$ & $0.085$ && $47.8$ & $0.064$  \\ 
  kPU &&  $52.1$ & $0.373$ && $50.2$ & $0.645$ && $35.5$ & $0.769$  \\ 
  PULSE (Ours) && ${62.2}$ & ${0.061}$ && $71.7$ & $0.044$ && $75.73$ & $0.179$  \\
  \bottomrule 
  \end{tabular}
\end{table}

\begin{table}[t]
  \centering
  \small
  \setlength{\tabcolsep}{1pt}
  \renewcommand{\arraystretch}{1.}
  \caption{\emph{Comparison of PULSE with other methods}. Across all datasets, PULSE outperforms alternatives for both target classification and novel class prevalence estimation. 
  Acc~(All) is target accuracy, Acc~(Seen) is target accuracy on examples from previously seen classes, and Acc~(Novel) is recall for novel examples. MPE~(Seen) is sum of absolute error for estimating target marginal among previously seen classes and  MPE~(Novel) is absolute error for novel prevalence estimation.  
  Results  reported by averaging across 3 seeds.}\label{table:results_full}
  \vspace{4pt}
  \begin{tabular}{@{}*{13}{c}@{}}
  \toprule
   && \multicolumn{5}{c}{\textbf{CIFAR-10}} && \multicolumn{5}{c}{\textbf{CIFAR-100}}   \\
    \multirow{2}{*}{Method} &&   \multirow{2}{*}{  \parbox{1.cm}{\centering Acc (All)}} & \multirow{2}{*}{  \parbox{1.cm}{\centering Acc (Seen)}} & \multirow{2}{*}{  \parbox{1.cm}{\centering  Acc (Novel)}} & \multirow{2}{*}{  \parbox{1.cm}{\centering MPE (Seen)}} & \multirow{2}{*}{  \parbox{1.cm}{\centering MPE (Novel)}} & &  \multirow{2}{*}{  \parbox{1.cm}{\centering Acc (All)}} & \multirow{2}{*}{  \parbox{1.cm}{\centering Acc (Seen)}} & \multirow{2}{*}{  \parbox{1.cm}{\centering Acc (Novel)}} & \multirow{2}{*}{  \parbox{1.cm}{\centering MPE (Seen)}} & \multirow{2}{*}{  \parbox{1.cm}{\centering MPE (Novel)}} \\
    & & & & & & & & \\
  \midrule
  Source-Only && ${67.1}$ & $87.0$ & - & - & - && $46.6$ & $66.4$ & - & - & - \\
  \midrule 
  UAN~\citep{you2019universal}   && $15.4$ & $19.7$ & $25.2$ & $1.44$ & $0.214$  && $18.1$ & $40.6$ & $14.8$ & $1.48$ & $0.133$\\
  BODA~\citep{saito2018open} && $63.1$ & $66.2$ & $42.0$ & $0.541$ & $0.162$ && $36.1$ & $17.7$ & $81.6$ & $0.564$ & $0.41$\\ 
  DANCE~\citep{saito2020universal} && $70.4$ & $85.5$ & $14.5$ & $0.784$ & $0.174$ && $47.3$ & $66.4$ & $1.2$ & $0.702$ & $0.28$  \\
  STA~\citep{liu2019separate} && $57.9$ & $69.6$ & $14.9$ & $0.409$ & $0.124$ && $42.6$ & $48.5$ & $34.8$ & $0.798$ & $0.14$\\
  CMU~\citep{fu2020learning} && $62.1$ & $77.9$ & $41.2$ & $0.443$ & $0.183$ && $35.4$ & $46.0$ & $15.5$ & $0.695$ & $0.161$\\
  \midrule
  Domain Disc. && $47.4$ & $87.0$ & $30.6$ & - & $0.331$ && $45.8$ & $66.5$ & $39.1$ & - & $\bf{0.046}$\\
  $k$-PU && $83.6$ & $79.4$ & $\bf{98.9}$ & $\bf{0.062}$ & $0.036$ && $36.3$ & $22.6$ & $\bf{99.1}$ & $6.31$ & $0.298$  \\ 
  PULSE (Ours) && $\bf{86.1}$ & $\bf{91.8}$ &  $88.4$ & ${0.091}$ & $\bf{0.008}$ && $\bf{63.4}$ & $\bf{67.2}$ & $63.5$ & $\bf{0.365}$ & $0.078$ \\ 
  \bottomrule 
  \end{tabular}  
  
  \vspace{5pt}

  \begin{tabular}{@{}*{13}{c}@{}}
  \toprule
   && \multicolumn{5}{c}{\textbf{Entity30}} && \multicolumn{5}{c}{\textbf{Newsgroup20}}   \\
    \multirow{2}{*}{Method} &&   \multirow{2}{*}{  \parbox{1.cm}{\centering Acc (All)}} & \multirow{2}{*}{  \parbox{1.cm}{\centering Acc (Seen)}} & \multirow{2}{*}{  \parbox{1.cm}{\centering  Acc (Novel)}} & \multirow{2}{*}{  \parbox{1.cm}{\centering MPE (Seen)}} & \multirow{2}{*}{  \parbox{1.cm}{\centering MPE (Novel)}} & &  \multirow{2}{*}{  \parbox{1.cm}{\centering Acc (All)}} & \multirow{2}{*}{  \parbox{1.cm}{\centering Acc (Seen)}} & \multirow{2}{*}{  \parbox{1.cm}{\centering Acc (Novel)}} & \multirow{2}{*}{  \parbox{1.cm}{\centering MPE (Seen)}} & \multirow{2}{*}{  \parbox{1.cm}{\centering MPE (Novel)}} \\
    & & & & & & & & \\
  \midrule
  Source-Only && ${32.0}$ & $53.5$ & - & - & - && $39.3$ & $64.4$ & - & - & - \\
  \midrule 
  BODA~\citep{saito2018open} && $42.22$ & $25.9$ & $67.2$ & $0.367$ & $0.189$ && $43.4$ & $38.0$ & $34.1$ & $0.550$ & $0.167$\\ 
  \midrule
  Domain Disc. && $43.2$ & $53.5$ & $68.0$ & - & $0.135$ && $50.9$ & $64.4$ & $\bf{93.2}$ & - & $0.176$\\
  $k$-PU && $50.7$ & $22.3$ & $\bf{94.4}$ & ${0.99}$ & $0.394$ && $52.1$ & $57.8$ & $42.7$ & $0.776$ & $0.373$  \\ 
  PULSE (Ours) && $\bf{58.0}$ & $\bf{54.3}$ &  $72.2$ & $\bf{0.215}$ & $\bf{0.054}$ && $\bf{62.2}$ & $\bf{65.0}$ & ${83.6}$ & $\bf{0.232}$ & $\bf{0.061}$ \\ 
  \bottomrule 
  \end{tabular}  
  
    \vspace{5pt}
    
  \begin{tabular}{@{}*{13}{c}@{}}
  \toprule
   && \multicolumn{5}{c}{\textbf{Tabula Muris}} && \multicolumn{5}{c}{\textbf{BreakHis}}   \\
    \multirow{2}{*}{Method} &&   \multirow{2}{*}{  \parbox{1.cm}{\centering Acc (All)}} & \multirow{2}{*}{  \parbox{1.cm}{\centering Acc (Seen)}} & \multirow{2}{*}{  \parbox{1.cm}{\centering  Acc (Novel)}} & \multirow{2}{*}{  \parbox{1.cm}{\centering MPE (Seen)}} & \multirow{2}{*}{  \parbox{1.cm}{\centering MPE (Novel)}} & &  \multirow{2}{*}{  \parbox{1.cm}{\centering Acc (All)}} & \multirow{2}{*}{  \parbox{1.cm}{\centering Acc (Seen)}} & \multirow{2}{*}{  \parbox{1.cm}{\centering Acc (Novel)}} & \multirow{2}{*}{  \parbox{1.cm}{\centering MPE (Seen)}} & \multirow{2}{*}{  \parbox{1.cm}{\centering MPE (Novel)}} \\
    & & & & & & & & \\
  \midrule
  Source-Only && ${33.8}$ & $93.3$ & - & - & - && $70.0$ & $95.8$ & - & - & - \\
  \midrule 
  BODA~\citep{saito2018open} && $76.5$ & $59.8$ & $87.0$ & $0.200$ & $0.079$ && $71.5$ & $81.8$ & $44.0$ & $0.163$ & $0.077$\\ 
  \midrule
  Domain Disc. && $73.0$ & $93.3$ & $\bf{94.7}$ & - & $0.071$ && $56.5$ & $95.8$ & $\bf{90.4}$ & - & $0.09$\\
  $k$-PU && $85.9$ & $91.6$ & $83.3$ & $\bf{0.279}$ & $0.307$ && $75.6$ & $71.7$ & ${86.1}$ & ${0.094}$ & $0.058$  \\ 
  PULSE (Ours) && $\bf{87.8}$ & $\bf{94.6}$ &  $88.8$ & ${0.388}$ & $\bf{0.058}$ && $\bf{79.1}$ & $\bf{96.1}$ & ${76.3}$ & $\bf{0.090}$ & $\bf{0.054}$ \\ 
  \bottomrule 
  \end{tabular}  
  \vspace{5pt}
 
  \begin{tabular}{@{}*{8}{c}@{}}
  \toprule
   && \multicolumn{5}{c}{\textbf{Dermnet}}    \\
    \multirow{2}{*}{Method} &&   \multirow{2}{*}{  \parbox{1.cm}{\centering Acc (All)}} & \multirow{2}{*}{  \parbox{1.cm}{\centering Acc (Seen)}} & \multirow{2}{*}{  \parbox{1.cm}{\centering  Acc (Novel)}} & \multirow{2}{*}{  \parbox{1.cm}{\centering MPE (Seen)}} & \multirow{2}{*}{  \parbox{1.cm}{\centering MPE (Novel)}} \\
    & & & & \\
  \midrule
  Source-Only && ${41.4}$ & $53.6$ & - & - & -  \\
  \midrule 
  BODA~\citep{saito2018open} && $43.8$ & $31.4$ & $58.4$ & $\bf{0.401}$ & $0.207$\\ 
  \midrule
  Domain Disc. && $40.6$ & $53.6$ & $82.7$ & - & $0.083$ \\
  $k$-PU && $46.0$ & $26.0$ & $\bf{89.9}$ & ${1.44}$ & $0.313$  \\ 
  PULSE (Ours) && $\bf{48.9}$ & $\bf{53.7}$ &  $57.7$ & $\bf{0.41}$ & $\bf{0.043}$  \\ 
  \bottomrule 
  \end{tabular}  
  \vspace{-5pt}
\end{table}

\begin{figure}[t]
    \centering
    \subfigure[CIFAR10 ]{\includegraphics[width=0.3\linewidth]{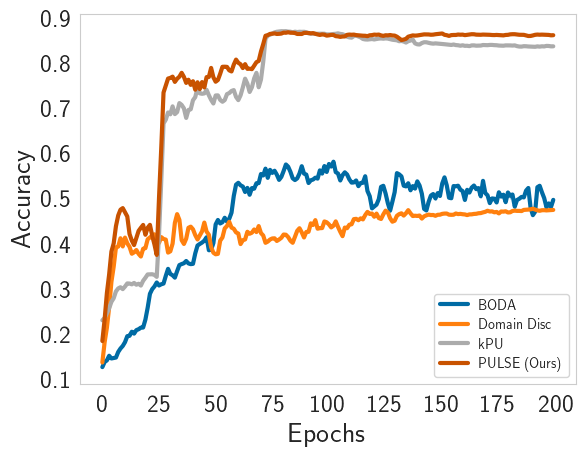}} \hfil
    \subfigure[CIFAR100  ]{\includegraphics[width=0.3\linewidth]{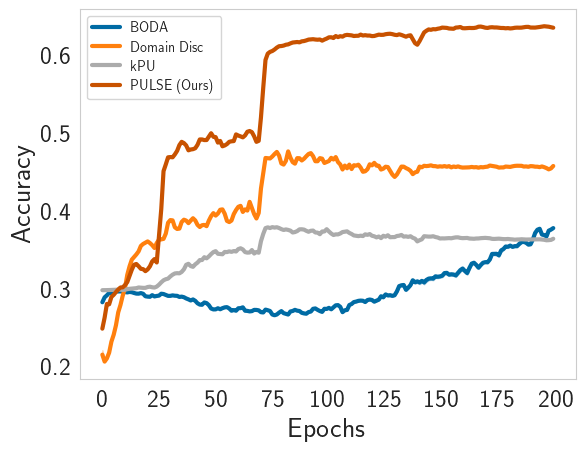}} \hfil
    \subfigure[Entity30 ]{\includegraphics[width=0.3\linewidth]{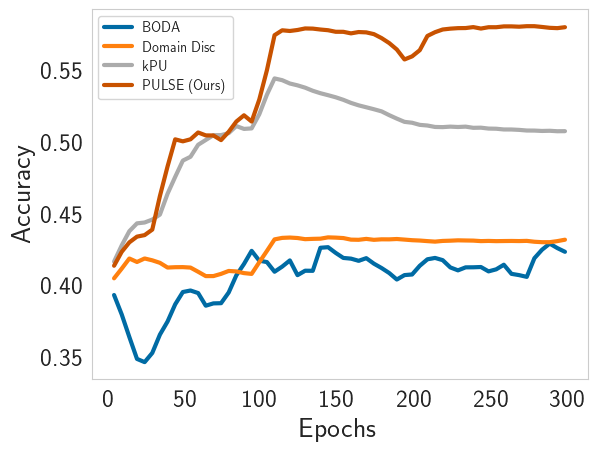}} \hfil
    \subfigure[Newsgroups20 ]{\includegraphics[width=0.3\linewidth]{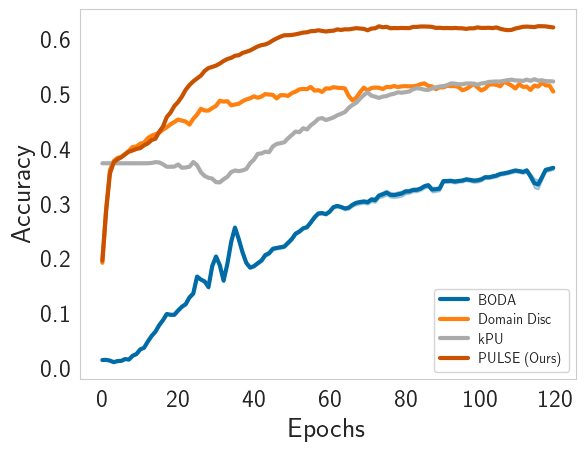}} \hfil
    \subfigure[Tabula Muris]{\includegraphics[width=0.3\linewidth]{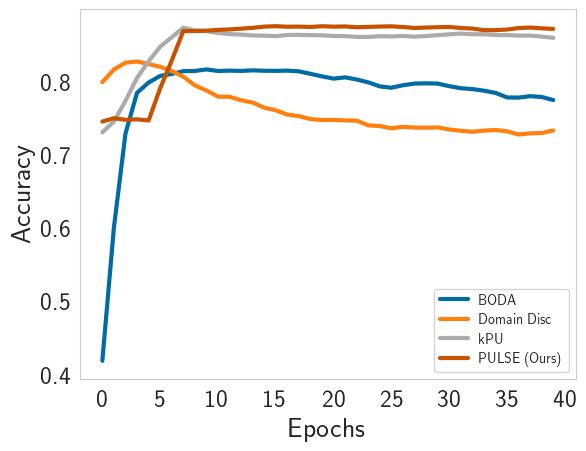}} \hfil
    \subfigure[BreakHis ]{\includegraphics[width=0.3\linewidth]{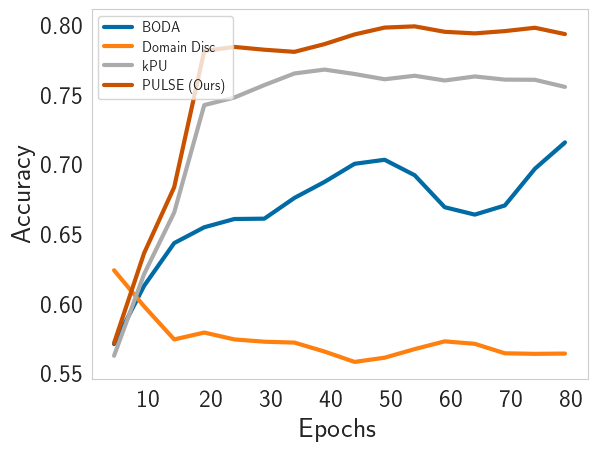}} \hfil
    \subfigure[Dermnet ]{\includegraphics[width=0.3\linewidth]{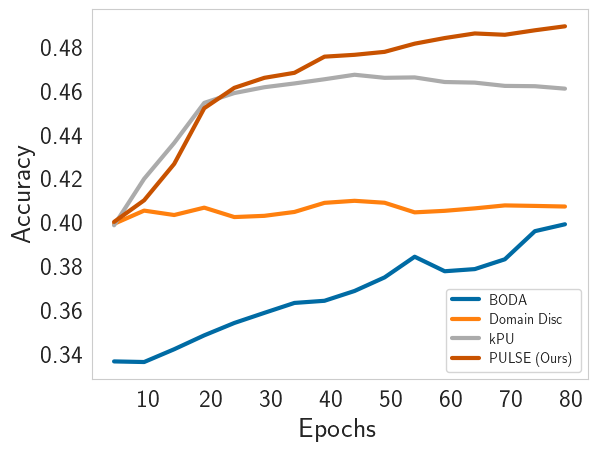}} \hfil
    
    \caption{\textbf{Epoch wise results for target accuracy.} Results aggregated over 3 seeds. PULSE maintains stable and superior performance when compared to alternative methods. } 
    \label{fig:accuracy_plot}
\end{figure}

\begin{figure}[t]
    \centering
    \subfigure[CIFAR10 ]{\includegraphics[width=0.3\linewidth]{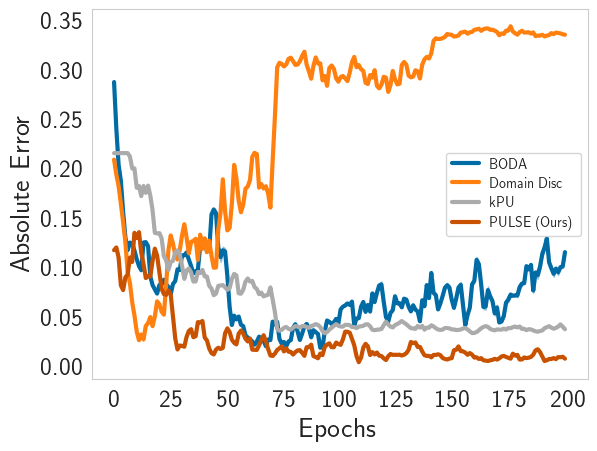}} \hfil
    \subfigure[CIFAR100  ]{\includegraphics[width=0.3\linewidth]{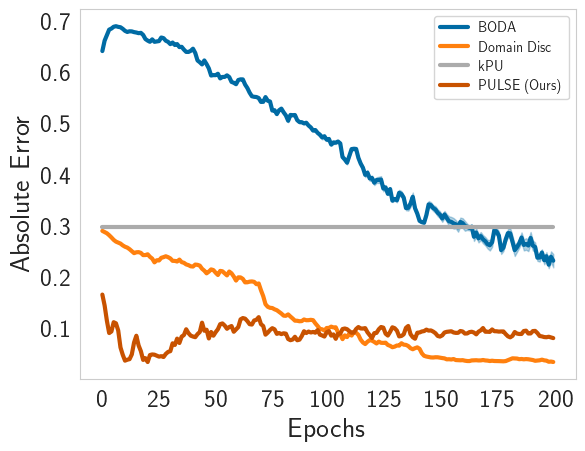}} \hfil
    \subfigure[Entity30 ]{\includegraphics[width=0.3\linewidth]{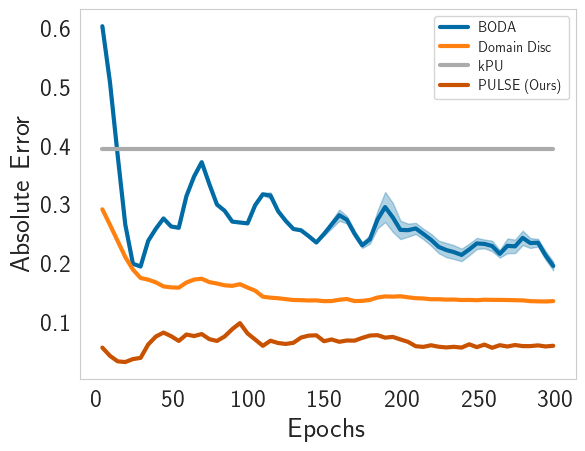}} \hfil
    \subfigure[Newsgroups20 ]{\includegraphics[width=0.3\linewidth]{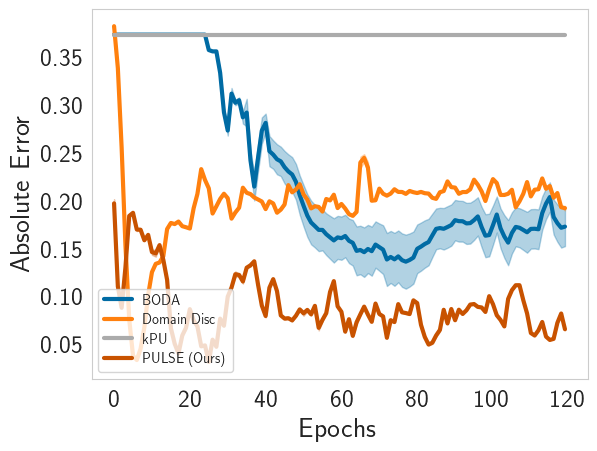}} \hfil
    \subfigure[Tabula Muris]{\includegraphics[width=0.3\linewidth]{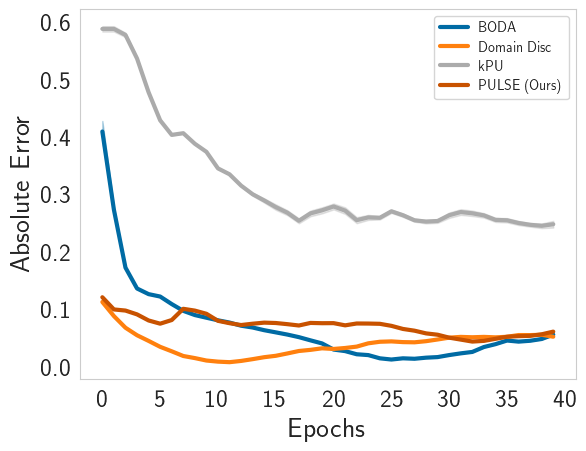}} \hfil
    \subfigure[BreakHis ]{\includegraphics[width=0.3\linewidth]{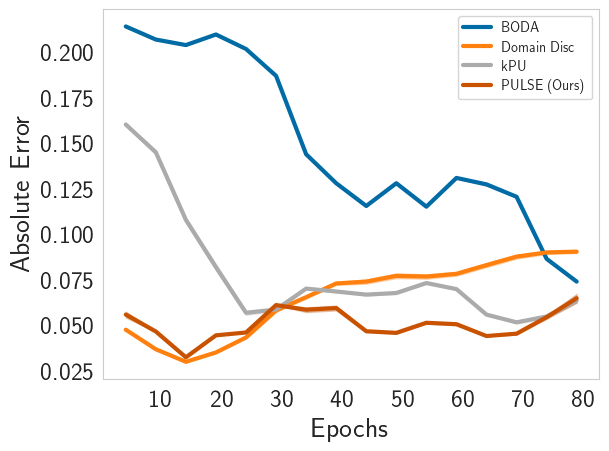}} \hfil
    \subfigure[Dermnet ]{\includegraphics[width=0.3\linewidth]{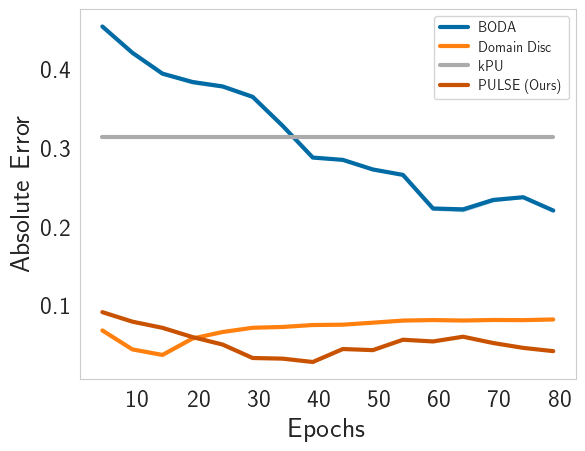}} \hfil
    
    \caption{\textbf{Epoch wise results for novel prevalence estimation.} Results aggregated over 3 seeds. PULSE maintains stable and superior performance when compared to alternative methods. } 
    \label{fig:mpe_plot}
\end{figure}

\update{\subsection{Age Prediction Task}}

\update{We consider an experiment on UTK Face dataset\footnote{\url{https://susanqq.github.io/UTKFace/}}.
We create an 8-way class classification problem where we split the age in the following 8 groups: $0$--$10$, $11$--$20$, $\cdots$, $60$--$70$ and $>70$. We consider the first 7 age groups in source and introduce age group $>70$ into the target data. OSLS continues to outperform the $k$PU baseline for novel prevalence estimation. Additionally, for target classification  performance of OSLS is similar to k$PU$ baseline (ref. \tabref{table:results_age_prediction}).
}

\update{
\begin{table}[h]
  \centering
  \small
  \setlength{\tabcolsep}{1.5pt}
  \renewcommand{\arraystretch}{1.2}
  \caption{Results on age prediction dataset. We observe that the prevalence of the novel class as estimated with our PULSE framework is significantly closer to the true estimate. Additionally target classification performance of OSLS is similar to that of $k$PU both of which significantly improve over domain discriminator and source only baselines.  }\label{table:results_age_prediction}
   \vspace{5pt}
  \begin{tabular}{@{}*{9}{c}@{}}
  \toprule
   &&    
   \multicolumn{2}{c}{\textbf{UTK Face}}  &  \\
    \multirow{2}{*}{Method} 
     &&    \multirow{2}{*}{  \parbox{1.0cm}{\centering  Acc (All)}}  &  \multirow{2}{*}{  \parbox{1.cm}{\centering MPE (Novel)}} 
      &  \\
      & & &  \\
  \midrule
  Source Only && $50.1$ & $0.11$ \\ 
  Domain Disc. && $52.4$ & $0.08$     \\ 
  kPU && $56.7$ & ${0.11}$   \\ 
  PULSE (Ours) && $56.8$ & ${0.01}$   \\
  \bottomrule 
  \end{tabular}  
  \vspace{-5pt}
\end{table}}

\end{document}